\documentclass[lineno]{biometrika}

\usepackage{amsmath, amssymb, amsxtra, graphicx, float, caption, subcaption, color, paralist}

\usepackage{times}  
\usepackage{bm}
\usepackage{natbib}

\usepackage[plain,noend]{algorithm2e}

\makeatletter
\renewcommand{\algocf@captiontext}[2]{#1\algocf@typo. \AlCapFnt{}#2} 
\def\@algocf@capt@plain{top}
\renewcommand{\algocf@makecaption}[2]{%
	\addtolength{\hsize}{\algomargin}%
	\sbox\@tempboxa{\algocf@captiontext{#1}{#2}}%
	\ifdim\wd\@tempboxa >\hsize
	\hskip .5\algomargin%
	\parbox[t]{\hsize}{\algocf@captiontext{#1}{#2}}
	\else%
	\global\@minipagefalse%
	\hbox to\hsize{\box\@tempboxa}
	\fi%
	\addtolength{\hsize}{-\algomargin}%
}
\makeatother

\newcommand{\inftt}{{2,\infty}}

\newcommand{\pr}{{\rm pr}}
\newcommand{\ep}{E}
\newcommand{\calN}{{\cal N}}
\newcommand{\ti}{{\tilde{i}}}
\newcommand{\tip}{{\tilde{i'}}}
\newcommand{\calF}{{\cal F}}

\newcommand{\theerrorrate}{{(n^{-1}\log n)^{1/2}}}
\newcommand{\sumiprime}{\sum_{i'\in \calN_i}}

\newcommand{\tc}{\tilde{C}}

\newcommand{\tblue}[1]{{\textcolor{black}{#1}}}

\def\twoImages#1#2#3#4#5
{
	\centerline{\hfill
		\includegraphics[width=#1]{#2}
		\includegraphics[width=#1]{#4}
		\hfill
	}
	\centerline{\hfill\makebox[#1]{#3}\hspace{-4em}\makebox[#1]{#5}\hfill\medskip}
}

\def\fourImages#1#2#3#4#5#6#7#8#9
{
	\centerline{\hfill
		\includegraphics[width=#1]{#2}
		\hfill
		\includegraphics[width=#1]{#4}
		\hfill
		\includegraphics[width=#1]{#6}
		\hfill
		\includegraphics[width=#1]{#8}
		\hfill
		\makebox[0.071\textwidth]{}
		\hfill
	}
	\centerline{\hspace{-0.02\textwidth}\hfill\makebox[#1]{#3}\hspace{-0.01\textwidth}\hfill\makebox[#1]{#5}\hspace{-0.01\textwidth}\hspace{-0.005\textwidth}\hfill\makebox[#1]{#7}\hfill\makebox[#1]{#9}\hfill\makebox[0.071\textwidth]{\ }\medskip}
}

\def\fourImagesWithScale#1#2#3#4#5#6#7#8#9
{
	\centerline{\hfill
		\includegraphics[width=#1]{#2}
		\hfill
		\includegraphics[width=#1]{#4}
		\hfill
		\includegraphics[width=#1]{#6}
		\hfill
		\includegraphics[width=#1]{#8}
		\hfill
		\includegraphics[width=0.071\textwidth, height=0.12\textheight]{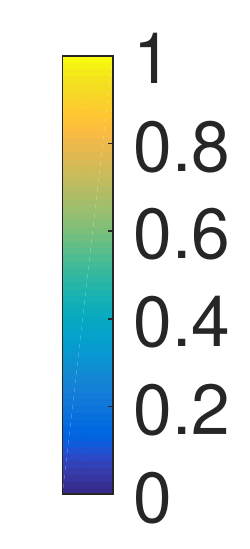}
		\hfill
	}
	\centerline{\hfill\makebox[#1]{#3}\hfill\makebox[#1]{#5}\hfill\makebox[#1]{#7}\hfill\makebox[#1]{#9}\hspace{0.03\textwidth}\hfill\makebox[0.071\textwidth]{\ }\medskip}
}

\def\fourImagesOnlyWithScale#1#2#3#4#5
{
	\centerline{\hfill
		\includegraphics[width=#1]{#2}
		\hfill
		\includegraphics[width=#1]{#3}
		\hfill
		\includegraphics[width=#1]{#4}
		\hfill
		\includegraphics[width=#1]{#5}
		\hfill
		\includegraphics[width=0.071\textwidth, height=0.12\textheight]{./colorbar.pdf}
		\hfill
	}
}

\begin{document}

\jname{\textcolor{red}{Submitted to Biometrika}}
\jyear{}
\jvol{}
\jnum{}
\copyrightinfo{}

\markboth{Zhang et al.}{Network edge probabilities by neighborhood smoothing}

\title{\vspace{-2em}Estimating network edge probabilities by neighborhood smoothing}

\author{Yuan Zhang}
\affil{Department of Statistics, Ohio State University\\404 Cockins Hall, 1958 Neil Avenue, Columbus, Ohio, U.S.A\\ \email{yzhanghf@stat.osu.edu}}
\author{Elizaveta Levina \and Ji Zhu}
\affil{Department of Statistics, University of Michigan\\ 311 West Hall, 1085 South University Avenue, Ann Arbor, Michigan, U.S.A.\\ \email{elevina@umich.edu}, \email{jizhu@umich.edu}}

\maketitle

\begin{abstract}

	The estimation of probabilities of network edges from the observed adjacency matrix has important applications to predicting missing links and network denoising.  It has usually been addressed by estimating the graphon, a function that determines the matrix of edge probabilities, but this is ill-defined without strong assumptions on the network structure.    Here we propose a novel computationally efficient method, based on neighborhood smoothing to estimate the expectation of the adjacency matrix directly, without making the structural assumptions that graphon estimation requires.  The neighborhood smoothing method requires little tuning, has a competitive mean-squared error rate, and outperforms many benchmark methods on link prediction in simulated and real networks.   
\end{abstract}


\section{Introduction}\label{section::introduction}

Statistical network analysis spans a wide range of disciplines, including network science, statistics, physics, computer science and sociology, and an equally wide range of applications and analysis tasks such as community detection and link prediction. In this paper, we study the problem of inferring the generative mechanism of an undirected network based on a single realization of the network. The data consist of the network adjacency matrix $A\in\{0,1\}^{n\times n}$, where $n$ is the number of nodes, and $A_{ij}=A_{ji}=1$ if there is an edge between nodes $i$ and $j$.  We assume the observed adjacency matrix $A$ is generated from an underlying probability matrix $P$, so that for $i\leq j$, $A_{ij}$'s are independent Bernoulli$(P_{ij})$ trials, and the $P_{ij}$ are edge probabilities. 

It is impossible to estimate $P$ from a single realization of $A$ unless one assumes some form of structure in $P$.    When the network is expected to have communities, arguably the most popular assumption is that of the stochastic block model, where each node belongs to one of $K$ blocks and the probability of an edge between two nodes is determined by the block to which the nodes belong.   In this case, the $n \times n$ matrix $P$ is parametrized by the $K \times K$ matrix of within- and between-block edge probabilities, and thus it is possible to estimate $P$ from a single realization.   The main challenge in fitting the stochastic block model is in estimating the blocks themselves, and that has been the focus of the literature, see for example \citet{bickel2009nonparametric}, \citet{rohe2011spectral}, \citet{amini2013pseudo}, \citet{saade2014spectral} and \citet{guedon2014community}.
Once the blocks are estimated, $P$ can be estimated efficiently by a plug-in moment estimator.    Many extensions and alternatives to the stochastic block model have been proposed to model networks with communities, including those of \citet{hoff2008modeling}, \citet{airoldi2009mixed}, \citet{karrer2011stochastic}, \citet{cai2015robust} and Zhang et al. (arXiv:1412.3432), but their properties are generally only known under the correctly specified model with communities. Here we are interested in estimating $P$ for more  general networks.

A general representation for the matrix $P$ for unlabeled exchangeable networks goes back to \citet{aldous1981representations} and the 1979 preprint by D. N. Hoover entitled ``Relations on probability spaces and arrays of random variables''. Formally, a network is exchangeable if for any permutation $\pi$ of the set $\{1, \dots, n\}$, the distribution of edges remains the same.    That is, if the adjacency matrix $A = [A_{ij}]$ is drawn from the probability matrix $P$, which we write as $A \sim P$, then for any permutation $\pi$,  
\begin{equation}
\left[A_{\pi(i)\pi(j)} \right]  \sim P  \ .  \label{introduction::exchangeability}
\end{equation}
\citet{aldous1981representations} and Hoover showed that an exchangeable network always admits the following Aldous--Hoover representation:
\begin{definition}
	For any network satisfying \eqref{introduction::exchangeability}, there exists a function $f:[0,1]\times[0,1]\to[0,1]$ and a set of independent and identically distributed random variables $\xi_i\sim \textrm{Uniform}[0, 1]$, such that
	\begin{equation}
	P_{ij} = f(\xi_i, \xi_j)  \ .  \label{introduction::AldousHooverrepresentation}
	\end{equation}
\end{definition}

Following the literature, we call $f$ the graphon function.  Unfortunately, as pointed out in Diaconis \& Jason (arXiv 0712.2749), $f$ in this representation is neither unique nor identifiable, since 
for any measure-preserving one-to-one transformation $\sigma:[0,1]\to[0,1]$, both $f\{\sigma(u), \sigma(v)\}$ and $f(u, v)$ yield the same distribution of $A$.    An identifiable and unique canonical representation can be defined if one requires $g(u) = \int_0^1f(u, v)\textrm{d}v$ to be non-decreasing  \citep{bickel2009nonparametric}.   \citet{chan2014consistent} show that $f$ and $\xi_i$'s are jointly identifiable when $g(u)$, which can be interpreted as expected node degree, is strictly monotone.  This assumption is strong and excludes the stochastic block model.

In practice, the main purpose of estimating $f$ is to estimate $P$, and thus identifiability of $f$ or lack thereof may not matter if $P$ itself can be estimated.  The preprint by Hoover and  Diaconis \& Jason (arXiv 0712.2749) showed that the measure-preserving map $\sigma$ is the only source of non-identifiability.  Wolfe and Olhede (arXiv:1309.5936) and \citet{choi2014co}  proposed estimating $f$ up to a measure-preserving transformation $\sigma$ via step-function approximations based on fitting the stochastic block model with a larger number of blocks $K$.   This approximation does not assume that the network itself  follows the block model, and some theoretical guarantees have been obtained under more general models.    In related work, \citet{olhede2014network} proposed to approximate the graphon with so-called network histograms, that is, stochastic block models with many blocks of equal size, akin to histogram bins.   Another method to compute a network histogram was proposed by \cite{amini2014semidefinite}, as an application of their semi-definite programming approach to fitting block models with equal size blocks.    Recently,  \citet{gao2014rate} established the minimax error rate for estimating $P$ and proposed a least squares type estimator to achieve this rate, which obtains the estimated probability $P$ by averaging the adjacency matrix elements within a given block partition.    A similar estimator was proposed in \citet{choi2015co}, applicable also to non-smooth graphons.  However, these methods are in principle computationally infeasible since they require an exhaustive enumeration of all possible block partitions.   Cai et al. (arXiv:1412.2129) proposed an iterative algorithm to fit a stochastic blockmodel and approximate the graphon, but its error rate is unknown for general graphons.    A Bayesian approach using block priors proposed by Gao et al. (arXiv:1506.02174) achieves the minimax error rate adaptively, but it still requires the evaluation of the posterior likelihood over all possible block partitions to obtain the posterior mode or the expectation for the probability matrix.

Other recent efforts on graphon estimation focus on the case of monotone node degrees, which make the graphon identifiable.  The sort and smooth methods of \citet{yang2014nonparametric} and \citet{chan2014consistent} estimate the graphon under this assumption by first sorting nodes by their degrees and then smoothing the matrix $A$ locally to estimate edge probabilities.   The monotone degree assumption is crucial for the success of these methods, and as we later show, the sort and smooth methods perform poorly when it does not hold.  Finally, general matrix denoising methods can be applied to this problem if one considers $A$ to be a noisy version of its expectation $P$;   a good general representative of this class of methods is the universal singular value thresholding approach of \citet{chatterjee2014matrix}.  Since this is a general method, we cannot expect its error rate to be especially competitive for this specific problem, and indeed its mean squared error rate is slower than the cubic root of the minimax rate.

In this paper, we propose a novel computationally efficient method for edge probability matrix estimation based on neighborhood smoothing, for piecewise Lipschitz graphon functions.  The key to this method is adaptive neighborhood selection, which allows us to avoid making strong assumptions about the graphon.   A node's neighborhood consists of nodes with similar rows in the adjacency matrix, which intuitively correspond to nodes with similar values of the latent node positions $\xi_i$.  To the best of our knowledge, our estimator achieves the best error rate among existing computationally feasible methods;    it allows easy parallelization.   The size of the neighborhood is controlled by a tuning parameter, similar to bandwidth in nonparametric regression;  the rate of this bandwidth parameter is determined by theory, and we show empirically that the method is robust to the choice of the constant.      Experiments on synthetic networks demonstrate that our method performs very well under a wide range of graphon models, including those of low rank and full rank, with and without monotone degrees.  We also test its performance on the link prediction problem, using both synthetic and real networks.

\section{The neighborhood smoothing estimator and its error rate}\label{section::estimation}

\subsection{Neighborhood smoothing for edge probability estimation}\label{section::estimation::estimator}
Our goal is to estimate the probabilities $P_{ij}$ from the observed network adjacency matrix $A$, where each $A_{ij}$ is independently drawn from $\textrm{Bernoulli}(P_{ij})$.   While $P_{ij}= f(\xi_i, \xi_j)$, where $\xi_i$'s are latent, our goal is to estimate $P$ for the single realization of $\xi_i$'s that gave rise to the data, rather than the function $f$.     We think of $f$ as a fixed unknown smooth function on $[0,1]^2$, with formal smoothness assumptions to be stated later.     Let $e_{ij} = e_{ij}(P_{ij})$ denote the Bernoulli error and omit its dependence on $P$.   We can then write  
\begin{equation}
A_{ij} = P_{ij} + e_{ij} = f(\xi_i, \xi_j) + e_{ij} .  \label{equation::NonParaRegression}
\end{equation}
Formulation \eqref{equation::NonParaRegression} resembles a nonparametric regression problem, except that the $\xi_i$ are not observed.   This has important consequences: for example, assuming further smoothness in $f$ beyond order one does not improve the minimax error rate when estimating $P$ \citep{gao2014rate}.    Our approach is to apply neighborhood smoothing, which would be natural had the latent variables $\xi_i$'s been observed.   Intuitively, if we had a set $\calN_i$ of neighbors of a node $i$, in the sense that $\calN_i = \{i':  P_{i'\cdot} \approx P_{i \cdot}\}$, where  $P_{i\cdot}$ represents the $i$-th row of $P$, then we could estimate $P_{i\cdot}$ by averaging $A_{i'\cdot}$ over $i' \in \calN_i$.   Postponing the question of how to select $\calN_i$  until Section \ref{section::estimation::neighborhoodselection}, we define a general neighborhood smoothing estimator by 
\begin{equation}
\tilde{P}_{ij} = \frac{\sum_{i'\in \calN_i}A_{i'j}}{|\calN_i|}\ . \label{equation::asymmetric_estimator}
\end{equation}
When the network is symmetric, we instead use a symmetric estimator
\begin{equation}
\hat{P} = \left(\tilde{P}+\tilde{P}^T\right)/2\ . \label{equation::estimator}
\end{equation}
For simplicity, we focus on undirected networks.   A natural alternative is to average over $\calN_i\times\calN_j$, but \eqref{equation::asymmetric_estimator} and \eqref{equation::estimator} allow vectorization and are thus more computationally efficient.   Our estimator can also be viewed as a relaxation of step function approximations such as \citet{olhede2014network}. In step function approximations, the neighborhood for each node is the nodes from its block, so the neighborhoods for two nodes from the same block are very similar, and the blocks have to be estimated first;   in contrast, neighborhood smoothing provides for more flexible neighborhoods that differ from node to node, and an efficient way to select the neighborhood, which we will discuss next.

\subsection{Neighborhood selection}\label{section::estimation::neighborhoodselection}

Selecting the neighborhood $\calN_i$ in \eqref{equation::estimator} is the core of our method.  Since we estimate $P_{i\cdot}$ by averaging over $A_{i'\cdot}$ for $i'\in\calN_i$, good neighborhood candidates $i'$ should have $f(\xi_{i'}, \cdot)$ close to $f(\xi_i, \cdot)$, which implies $P_{i'\cdot}$ close to $P_{i\cdot}$. We use the $\ell_2$ distance between graphon slices to quantify this, defining 
\begin{equation}
d(i, i')=\|f(\xi_{i}, \cdot) - f(\xi_{i'}, \cdot)\|_2=\left\{  \int_0^1 \left| f(\xi_{i}, v) - f(\xi_{i'}, v) \right|^2 \textrm{d}v  \right\} ^{1/2} \ .  \label{equation::graphonslicedistance}
\end{equation}
While one may consider more general $\ell_p$ or other distances,  the $\ell_2$ distance is particularly easy to work with theoretically.    For the purpose of neighborhood selection, it is not necessary to estimate $d(i, i')$;  it suffices to provide a tractable upper bound.   For integrable functions $g_1$ and $g_2$ defined on $[0, 1]$, define $\langle g_1, g_2 \rangle = \int_0^1 g_1(u)g_2(u)\textrm{d}u$.  Then we can write 
\begin{align}
d^2(i, i') &= \langle  f(\xi_i, \cdot), f(\xi_i, \cdot)  \rangle   +   \langle  f(\xi_{i'}, \cdot), f(\xi_{i'}, \cdot)  \rangle - 2\langle  f(\xi_i, \cdot), f(\xi_{i'}, \cdot)  \rangle . 
\label{equation::decomposel2distance}
\end{align}
The third term in \eqref{equation::decomposel2distance} can be estimated by $2\langle A_{i\cdot}, A_{i'\cdot} \rangle / n$, where $A_{i\cdot}$ and $A_{i'\cdot}$ are nearly independent up to a single duplicated entry due to symmetry.     The first two terms in \eqref{equation::decomposel2distance} are more difficult, since $\langle A_{i\cdot}, A_{i\cdot} \rangle / n$ is not a good estimator for $\langle  f(\xi_i, \cdot), f(\xi_i, \cdot)  \rangle$.  Here we present the intuition and provide a full theoretical justification in Theorem \ref{theorem::errorrate}. For simplicity, assume for now $f$ is Lipschitz with a Lipschitz constant of $1$.
The idea is to use nodes with graphon slices similar to $i$ and $i'$ to make the terms in the inner product distinct graphon slices.  With high probability, for each $i$, we can find $\ti\neq i$ such that $|\xi_{\ti}-\xi_i|\leq  e_n$, where the sequence $e_n$ is a function of $n$ and represents the error rate to be specified later.   Then $\|f(\xi_i, \cdot) - f(\xi_\ti, \cdot)\|_2\leq e_n$, and we can approximate $\langle f(\xi_i, \cdot), f(\xi_i, \cdot) \rangle$ by $\langle f(\xi_i, \cdot), f(\xi_\ti, \cdot) \rangle$, where the latter can now be estimated by $\langle  A_{i\cdot}, A_{\ti\cdot}  \rangle / n$.   The same technique can be used to approximate the second term in \eqref{equation::decomposel2distance}, but all these approximations depend on the unknown $\xi$'s. To deal with this,  we rearrange the terms in \eqref{equation::decomposel2distance} as follows:
\begin{align}
d^2(i, i') &= \langle  f(\xi_i, \cdot) - f(\xi_{i'}, \cdot), f(\xi_i, \cdot)  \rangle    -   \langle  f(\xi_i, \cdot) - f(\xi_{i'}, \cdot), f(\xi_{i'}, \cdot)  \rangle \nonumber\\
& \leq \left| \langle  f(\xi_i, \cdot) - f(\xi_{i'}, \cdot), f(\xi_{\ti}, \cdot)  \rangle \right| + \left| \langle  f(\xi_i, \cdot) - f(\xi_{i'}, \cdot), f(\xi_{\tip}, \cdot)  \rangle \right|  +  2  e_n 
\nonumber\\
& \leq 2 \max_{k\neq i,i'}|\langle f(\xi_i, \cdot) - f(\xi_{i'}, \cdot), f(\xi_k, \cdot) \rangle| + 2  e_n \ .
\label{equation::controldsquared}
\end{align}
The inner product on the right side of \eqref{equation::controldsquared} can be estimated by
\begin{equation}
\tilde{d}^2(i, i')= \max_{k\neq i,i'}\left|\langle A_{i\cdot} - A_{i'\cdot}, A_{k\cdot} \rangle\right|\big/ n \ . 
\label{equation::defgraphonslicedistance}
\end{equation}
Intuitively, the neighborhood $\calN_i$ should consist of  $i'$s with small $\tilde{d}(i, i')$.   To formalize this, let $q_i(h)$ denote the $h$-th sample quantile of the set $\left\{ \tilde{d}(i, i'): i'\neq i \right\}$, where $h$ is a tuning parameter, and set 
\begin{equation}
\calN_i = \left\{ i'\neq i: \tilde{d}(i, i')\leq q_i(h) \right\}  \label{equation::neighborhoodadmission}
\end{equation}
where for notational simplicity we suppress  the dependence of $\calN_i$ on $h$.   Thresholding at a quantile rather than at some absolute value is convenient since real networks vary in their average node degrees and other parameters, which leads to very different values and distributions of $\tilde{d}$.  Empirically, thresholding at a quantile shows significant advantage in stability and performance compared to an absolute threshold.  The choice of $h$ will be guided by both the theory in Section \ref{section::estimation::errorrate}, which suggests  the order of $h$, and empirical performance which suggests the constant factor.  More details are included in the Supplementary Material. 

An important feature of this definition is that the neighborhood admits nodes with similar graphon slices, but not necessarily similar $\xi$'s.   For example, in the stochastic block model, all nodes from the same block would be equally likely to be included in each other's neighborhoods, regardless of their $\xi$'s.  Even though we use $\xi_i$ and $\xi_{i'}$ to motivate \eqref{equation::controldsquared}, we always work with the function values $f(\xi_i, \xi_j)$'s and never attempt to estimate the $\xi_i$ or $f$ by themselves. This contrasts with the approaches of \citet{chan2014consistent} and \citet{yang2014nonparametric}, and gives us a substantial computational advantage as well as much more flexibility in assumptions.

\subsection{Consistency of the neighborhood smoothing estimator}\label{section::estimation::errorrate}

We study the theoretical properties of our estimator for a family of piecewise Lipschitz graphon functions, defined as follows.

\begin{definition}[Piecewise Lipschitz graphon family]
	For any $\delta, L>0$, let $\calF_{\delta; L}$ denote a family of piecewise Lipschitz graphon functions $f: \ [0,1]^2 \rightarrow [0,1]$ such that
	\begin{inparaenum}[(i)]
		\item  there exists an integer $K\geq 1$ and a sequence $0=x_0 < \cdots < x_K=1$ satisfying $\min_{0\leq s \leq K-1}(x_{s+1} - x_s)\geq\delta$, and
		\item both $\left| f(u_1, v) - f(u_2, v) \right| \leq L|u_1-u_2|$ and $\left| f(u, v_1) - f(u, v_2) \right| \leq L|v_1-v_2|$ hold for all $u, u_1, u_2 \in [x_s, x_{s+1}]$, $v, v_1, v_2 \in [x_t, x_{t+1}]$ and $0\leq s,t \leq K-1$.
	\end{inparaenum}
\end{definition}

For any $P,Q\in\mathbb{R}^{m\times m}$, define $d_\inftt$, the normalized $\inftt$ matrix norm, by $$d_\inftt(P, Q) = m^{-1/2}\|P-Q\|_\inftt = \max_i m^{-1/2} \|P_{i\cdot}-Q_{i\cdot}\|_2 \ . $$  
Then we have the following error rate bound.

\begin{theorem}\label{theorem::errorrate}
	Assume that $L$ is a global constant and $\delta=\delta(n)$ depends on $n$, satisfying $\lim_{n\to\infty}\delta/\theerrorrate\to\infty$.   Then the estimator $\tilde{P}$ defined in \eqref{equation::estimator}, with neighborhood $\calN_i$ defined in \eqref{equation::neighborhoodadmission} and $h=C\theerrorrate$ for any global constant $C\in(0,1]$, satisfies
	\begin{equation}
	\max_{f\in \calF_{\delta; L}} \pr\left\{ d_\inftt(\tilde{P},P)^2 \geq C_1 \left(\frac{\log n}{n}\right)^{1/2} \right\} \leq n^{-C_2}  \label{equation::maintheorem}
	\end{equation}
	where $C_1$ and $C_2$ are positive global constants.
\end{theorem}

Since for any $P, Q\in \mathbb{R}^{m\times m}$, we have $d_\inftt(P, Q)\geq m^{-1}\|P-Q\|_F$, Theorem \ref{theorem::errorrate} yields
\begin{corollary}  Under conditions of  Theorem \ref{theorem::errorrate}, 
	\begin{equation}
	\max_{f\in \calF_{\delta; L}} \pr\left\{ \frac{1}{n^2}\|\tilde{P} - P\|_F^2 \geq C_1 \left(\frac{\log n}{n}\right)^{1/2} \right\} \leq n^{-C_2}  \ . 
	\label{equation::maintheorem::corollary}
	\end{equation}
\end{corollary}
The bound \eqref{equation::maintheorem::corollary} continues to hold if we replace $\tilde{P}$ by $\hat{P}$, but \eqref{equation::maintheorem} may not hold.  Next, we show that under the $(2,\infty)$ norm, our estimator $\tilde{P}$ is nearly rate-optimal, up to a $\log n$ factor.

\begin{theorem}\label{proposition:sliceminimax}
	Under conditions of Theorem \ref{theorem::errorrate}, we have
	\begin{equation}
	\inf_{\hat{P}}\sup_{f\in\mathcal{F}_{\delta; L}} E \left\{d_\inftt^2(\hat{P}, P)\right\}\geq 	C\left(n\log n\right)^{-1/2}
	\end{equation}
	for some global constant $C>0$.
\end{theorem}

To the best of our knowledge, result \eqref{equation::maintheorem} is the only $(2,\infty)$ error rate available for polynomial time graphon estimation methods.  Most previous work focused on the mean squared error and only considered the special case $\delta=1$.  For $\delta = 1$, the minimax error rate $\log n /n$ established by \citet{gao2014rate} has so far only been achieved by methods that require combinatorial optimization or evaluation, including \citet{gao2014rate} and \citet{klopp2015oracle}.  The rate $\left(\log n/n\right)^{1/2}$ was previously achieved by combinatorial methods, including 
Wolfe and Olhede (arXiv:1309.5936) and \citet{olhede2014network}.    Among computationally efficient methods, 
singular value thresholding (\citet{chatterjee2014matrix}, Theorem 2.7) achieves $n^{-1/3}$.   Additionally, the sort-and-smooth method proposed by \citet{chan2014consistent} achieves the minimax error rate under the strong assumption that $f$ has strictly monotone expected node degrees $d_f(v) = \int_0^1 f(u, v)\textrm{d}u$.  An anonymous referee of this manuscript sent us a proof that thresholding the leading $k$ singular values of the matrix $A$ achieves the mean squared error of $k/n+k^{-2}$, where the variance $k/n$ is due to \citet{candes2011tight} and $k^{-2}$ is the bias.  Taking $k=n^{1/3}$ gives the best known mean squared error rate of $n^{-2/3}$ for a computationally efficient algorithm.  For the graphon family $f\in \calF_{\delta; L}$ where $\delta/\theerrorrate\to\infty$ that we study, the $n^{-1/3}$ singular value thresholding method and our method achieve the same mean squared error rate.

For the case of general $\delta$, we can show that the minimax rate  of $\log n /n$ established by \citet{gao2014rate} still holds for the family $\calF_{\delta; L}$, in Proposition \ref{proposition:errorrate}; see the Supplementary Material.

\begin{proposition}\label{proposition:errorrate}
	Under conditions of Theorem \ref{theorem::errorrate},  when $\delta/(\log n/n)^{1/2}\to\infty$, there exists a global constant $C_3>0$ such that
	\[
	\inf_{\tilde{P}} \max_{f\in\calF_{\delta;L}}\ep \left\{ \frac{1}{n^2}\|\hat{P} - P\|_F^2 \right\} \asymp \frac{\log n}{n} \ .
	\]
\end{proposition}
Whether this minimax error rate can be achieved by a computationally efficient method remains an open question.

\section{Probability matrix estimation on synthetic networks}
\label{section::simulations}


In this section we evaluate the performance of our symmetric estimator \eqref{equation::estimator} on estimating the probability matrix for synthetic networks.   We generate the networks from the four graphons listed in Table \ref{table::fourgraphons}, selected to have different features in different combinations (monotone degrees, low rank, etc).   The corresponding probability matrices are pictured in the first column of Figure \ref{figure::simulation::benchmark::graphon_1} (lower triangular half).  All networks have $n = 2000$ nodes. 

\begin{table}[h!]
	\tbl{Synthetic graphons}{
		\begin{tabular}{clccc}\smallskip
			Graphon			& Function $f(u,v)$	& Monotone degrees	& Rank & Local structure\\
			1	& $k/(K+1)$ if $u, v \in ( (k-1)/K, k/K )$,	& Yes	& $\lfloor \log n \rfloor$	& No\\
			& $0\textrm{$\cdot$}3/(K+1)$ otherwise; $K=\lfloor\log n\rfloor$  & & & \\
			2	& $\sin\left\{5\pi(u+v-1) + 1\right\}/2+0\textrm{$\cdot$}5$	& No	& 3	& No\\
			3	& $1-\Big[ 1+ \exp\big\{ 15\big(0\textrm{$\cdot$}8|u-v|\big)^{4/5}-0\textrm{$\cdot$}1 \big\}\Big]^{-1}$	& No	& Full	& No\\
			4	& $\left(u^2+v^2\right)/3 \cos\left\{1/\left(u^2+v^2\right)\right\} + 0.15$	& No	& Full	& Yes
		\end{tabular}
	}\label{table::fourgraphons}
\end{table}

Additional empirical results in the Supplementary Material show that our method is robust to the choice of the constant factor $C$ in the bandwidth $h$, for simplicity, we set $C=1$ for the rest of this paper.    Here we focus on comparing to benchmarks.    From the general matrix denoising methods, we include the widely used method of universal singular value thresholding \citep{chatterjee2014matrix} and the $n^{1/3}$ leading singular value thresholding method suggested by a referee.   We also compare to the sort and smooth methods of \citet{chan2014consistent} and \citet{yang2014nonparametric}.  These methods differ only in that the latter one employs singular value thresholding to denoise the network as a pre-processing step.  Due to space constraints, we present both methods in Table \ref{table::MSEs} but only \citet{chan2014consistent} in figures, since they are visually very similar.  

We also inlcude two approximations based on fitting a stochastic block model, called network histograms by \citet{olhede2014network}.   One is the oracle stochastic block model, where the blocks are based on the true values of the latent $\xi_i$'s.  This cannot be done in practice, but we use it as the gold standard for a step-function approximation.   The feasible version of this is an approximation based on a stochastic block model with estimated blocks; we fit it by regularized spectral clustering 
\citep{chaudhuri2012spectral}. 
Any other algorithm for fitting the stochastic blockmodel can be used to estimate the blocks;  for example, \citet{olhede2014network} used a local updating algorithm initialized with spectral clustering to compute their network histograms.   Here we chose regularized spectral clustering because of its speed and good empirical performance.      For both approximations, we set the number of blocks to $n^{1/2}$, as in \citet{olhede2014network}. 

A recent as yet  unpublished method kindly shared with us by E. Airoldi
proposes a stochastic block model approximation, adapting the method of \citet{airoldi2013stochastic} to work with a single adjacency matrix. It uses a dissimilarity measure $\sum_{k\neq i, i'}\left| \langle A_{i\cdot}-A_{k\cdot}, A_{k\cdot} \rangle \right|$, which we considered before choosing \eqref{equation::defgraphonslicedistance} because it leads to a better guaranteed error rate.  Airoldi's method then builds blocks by starting with one not-yet-clustered node $i$ and including all nodes whose dissimilarity from $i$  is below a threshold $\Delta$ as neighbors.    We found that our strategy of thresholding by quantile instead of a fixed threshold is more efficient and stable, and the theoretical error rate is better for our method.    

We present the heatmaps of results for a single realization in Figure~\ref{figure::simulation::benchmark::graphon_1}, and the root mean squared errors and the mean absolute errors  of $\hat P$ in Table \ref{table::MSEs}.    While these two errors mostly agree on method ranking, the few cases where they disagree indicate whether the errors are primariy coming from a small number of poorly estimated entries or are more uniformly distributed throughout the matrix.  

\begin{figure}[h!]
	\centering
	\includegraphics[width=0.16\textwidth]{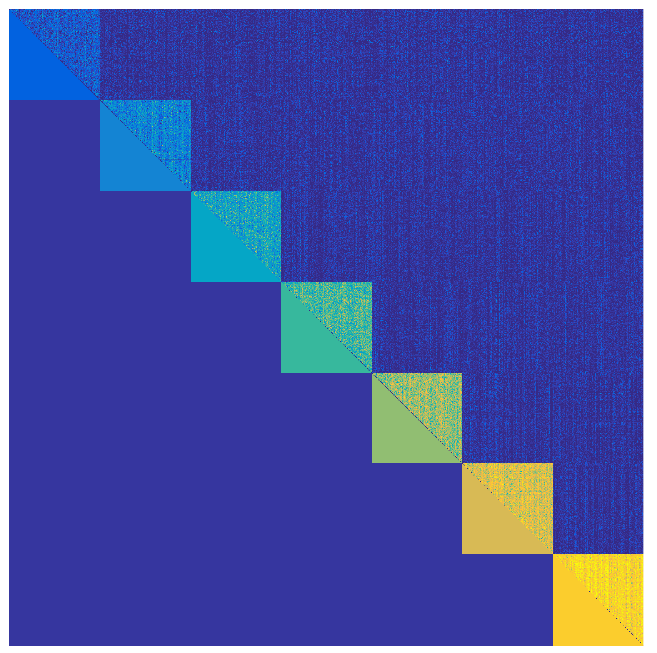}  
	\includegraphics[width=0.16\textwidth]{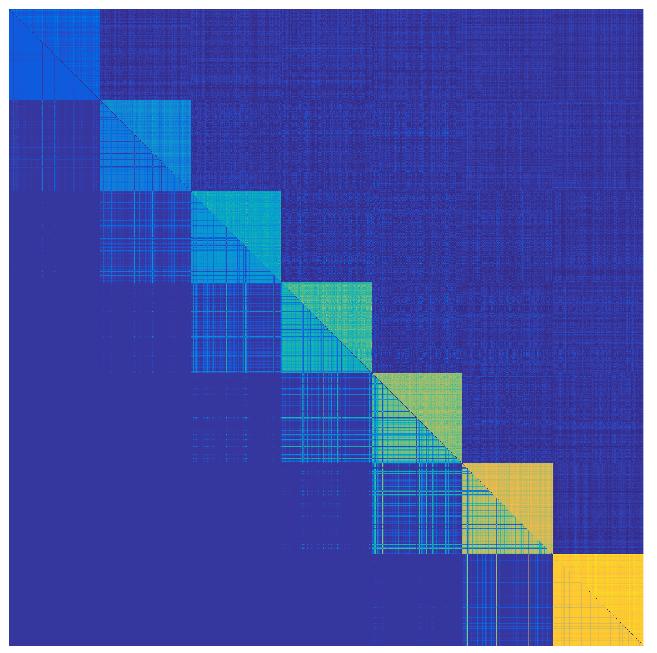} 
	\includegraphics[width=0.16\textwidth]{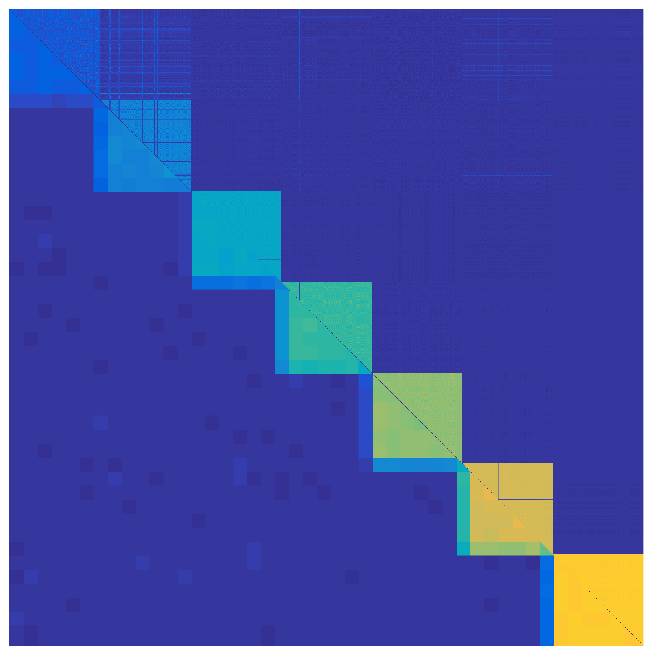} 
	\includegraphics[width=0.16\textwidth]{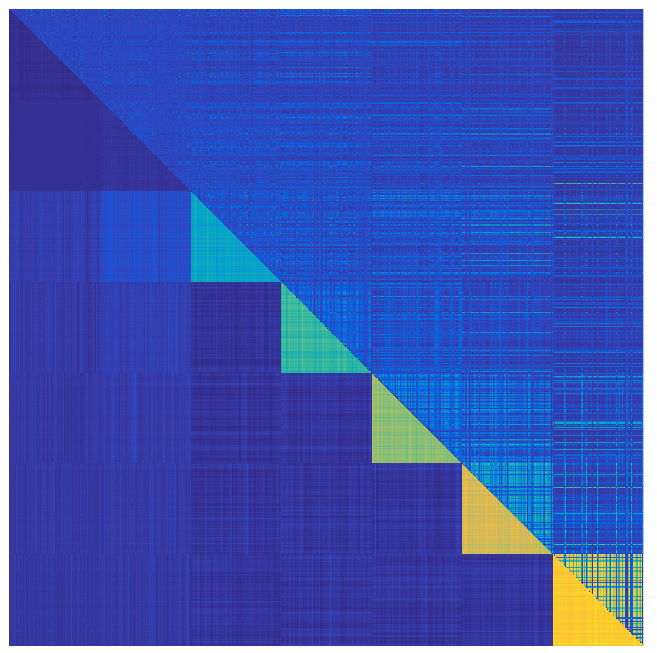} 
	\includegraphics[width=0.071\textwidth, height=0.11\textheight]{./colorbar.pdf}\\
	\includegraphics[width=0.16\textwidth]{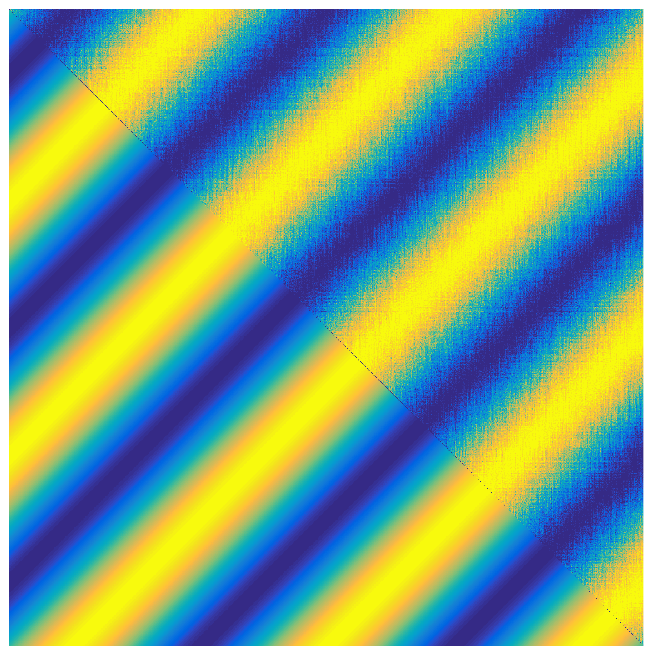}  
	\includegraphics[width=0.16\textwidth]{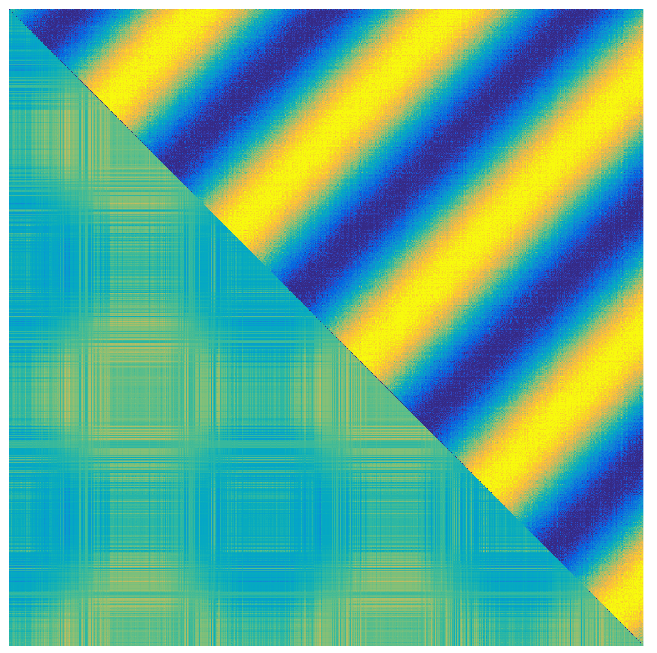} 
	\includegraphics[width=0.16\textwidth]{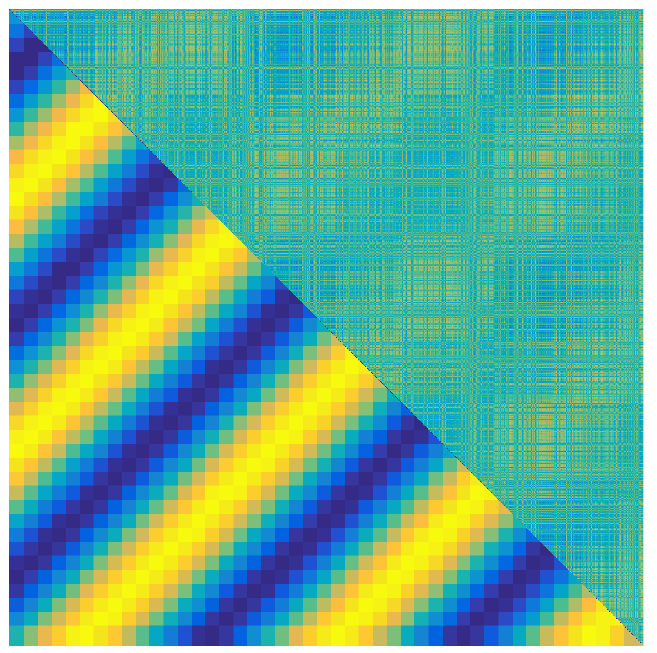} 
	\includegraphics[width=0.16\textwidth]{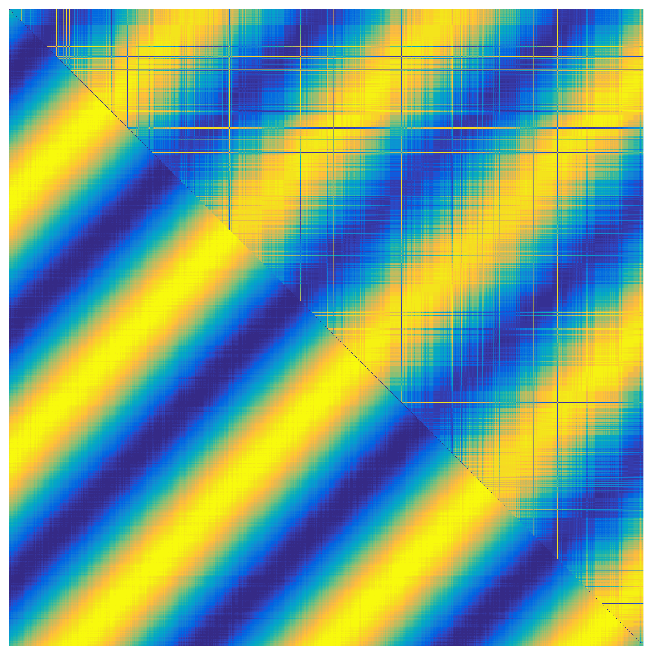}  
	\includegraphics[width=0.071\textwidth, height=0.11\textheight]{./colorbar.pdf}\\
	\includegraphics[width=0.16\textwidth]{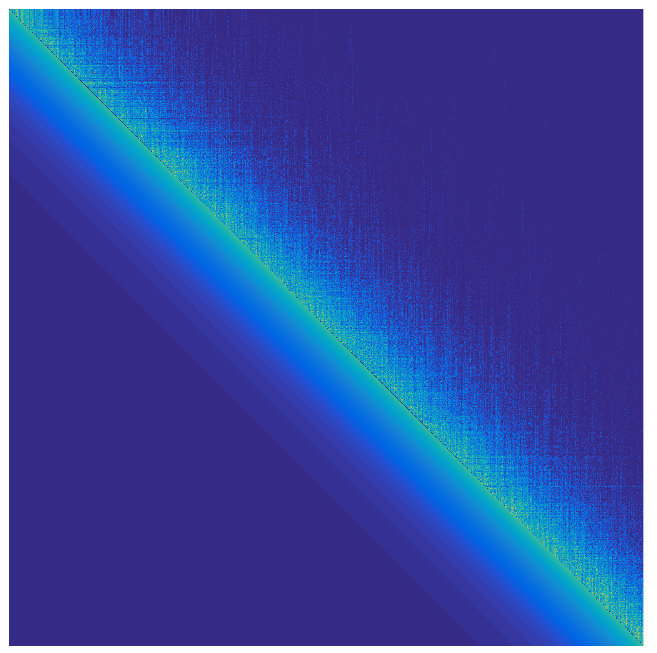} 
	\includegraphics[width=0.16\textwidth]{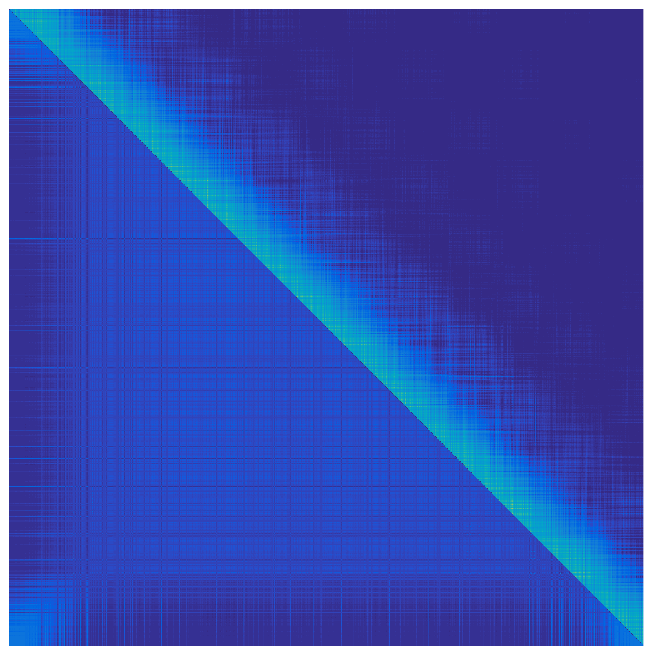} 
	\includegraphics[width=0.16\textwidth]{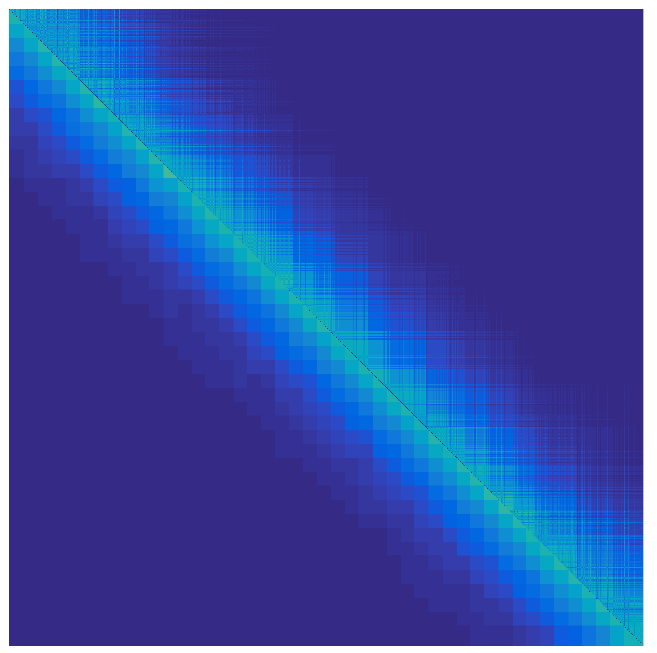} 
	\includegraphics[width=0.16\textwidth]{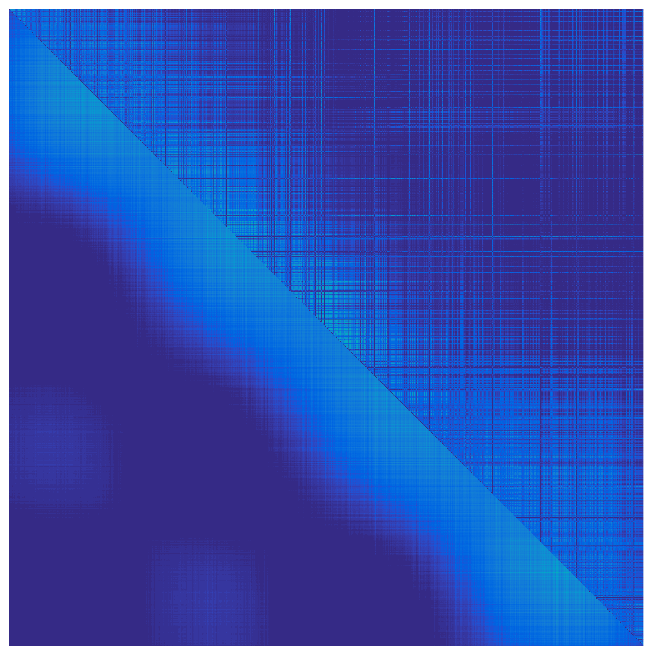} 
	\includegraphics[width=0.071\textwidth, height=0.11\textheight]{./colorbar.pdf}\\
	\includegraphics[width=0.16\textwidth]{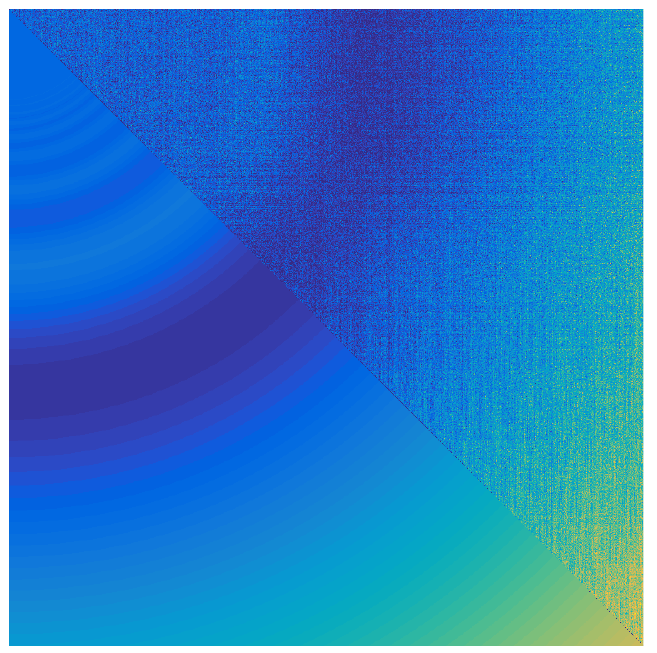} 
	\includegraphics[width=0.16\textwidth]{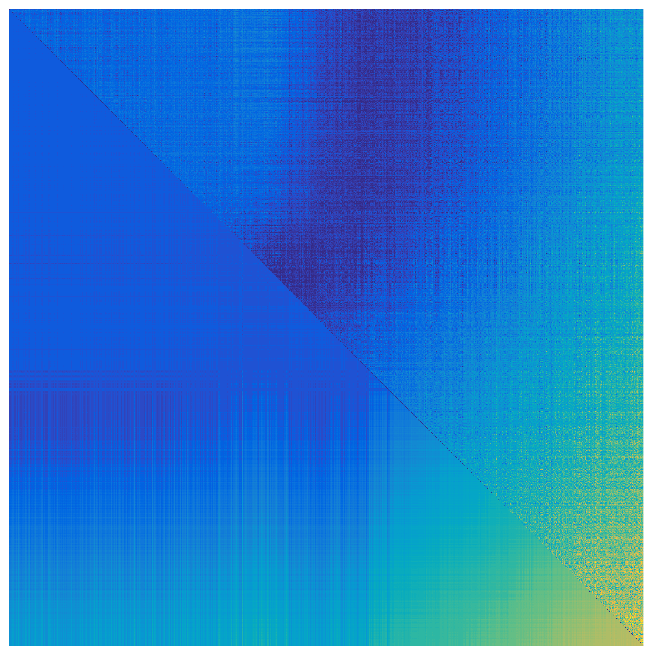} 
	\includegraphics[width=0.16\textwidth]{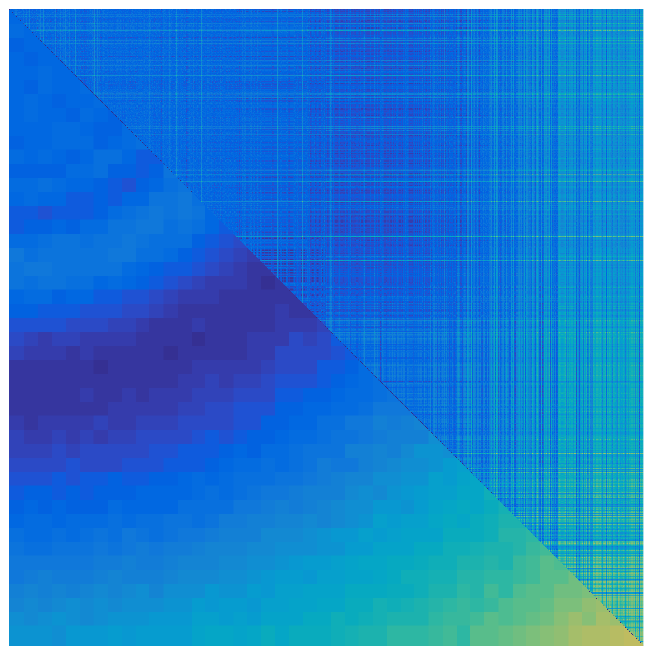} 
	\includegraphics[width=0.16\textwidth]{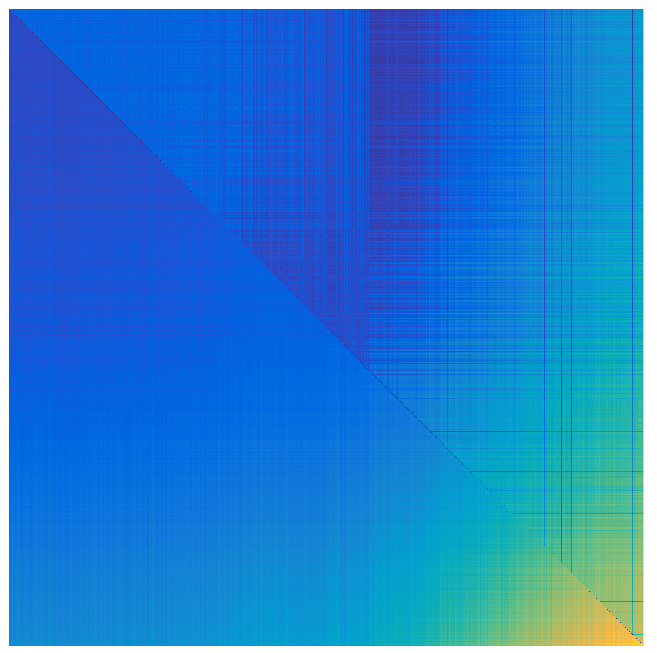} 
	\includegraphics[width=0.071\textwidth, height=0.11\textheight]{./colorbar.pdf}\\
	\caption{Estimated probability matrices for graphons 1--4, shown in rows 1--4.   Column 1: true $P$ (lower) and our method (upper).  Column 2: Chan \& Airoldi (2014) (lower) and $n^{1/3}$ singular value thresholding (upper).  Column 3:  Block model oracle (lower) and spectral clustering (upper).  Column 4: \citet{chatterjee2014matrix} (lower) and Airoldi's method (upper).}
	\vspace{-2em}
	\label{figure::simulation::benchmark::graphon_1}
\end{figure}

Grahpon 1 has $K  = \lfloor\log n\rfloor=7$ blocks with  different within-block edge probabilities, which all dominate the low between-block probability.  The best results are obtained by our method, singular value thresholding, spectral clustering, and the oracle stochastic blockmodel approximation, which is expected given that the data are generated from a stochastic block model.  The  oracle uses $n^{1/2}$ blocks rather than the true $K$, and thus makes substantial errors on the block boundaries, but not anywhere else.      
The method of \citet{chan2014consistent} correctly estimates the main blocks because they have different expected degrees, but suffers from boundary effects due to smoothing over the entire region.   In contrast, our method, which determines smoothing neighborhoods based on similarities of graphon slices,  does not suffer from boundary effects.  \citet{chatterjee2014matrix} does a good job on denser blocks but thresholds away sparser blocks.   Airoldi's method captures tightly connected communities, but does not do as well on weaker communities.

Graphon 2 lacks node degree monotonicity, and thus the method of 
\citet{chan2014consistent} does not work here.  Spectral clustering also performs poorly, likely because it uses too many ($n^{1/2}$) eigenvectors which add noise.  Airoldi's method and the stochastic block model oracle give grainy but reasonable approximations to $P$, and the best results are obtained by our method, \citet{chatterjee2014matrix}, and  singular value thresholding  with $n^{1/3}$ eigenvalues.  The latter two are expected to work well since this is a low-rank matrix.    

Graphon 3 is a diagonal-dominated matrix, and our method is the best among computationally efficient methods.  The method of \citet{chatterjee2014matrix} does not perform well because this is not a low-rank matrix;  spectral clustering, on the other hand, does fine, because there are many non-zero eigenvalues and the $n^{1/2}$ eigenvectors contain enough information.  The $n^{1/3}$ singular value thresholding does better than \citet{chatterjee2014matrix}  and provides a lower-resolution denoising.  Airoldi's method only roughly shows the structure, likely due to the similarity measure it uses.  The method of \citet{chan2014consistent} fails since all node expected degrees are almost the same.

Graphon 4
is difficult to estimate for all methods.  It is full rank, with structure at different scales.  This makes it a difficult setting for low-rank approximations, among which the $n^{1/3}$ singular value thresholding alone uses enough eigenvalues to produce a reasonable result, albeit with boundary effects.
This graphon is not a block matrix, and thus spectral clustering does not perform well.  The expected node degrees are not the same, but their ordering does not match the ordering of the latent node positions, so this graphon is also difficult for the sort-and-smooth method of \citet{chan2014consistent}.   Our method successfully picks up the global structure and the curvature.  While visually it is fairly similar to the result of $n^{1/3}$ singular value thresholding, our method has significantly better errors in Table \ref{table::MSEs}.   
Overall, this example illustrates a limitation of all global methods when there are subtle local differences.  

Table \ref{table::MSEs} shows the mean squared errors and the mean absolute errors of all methods on the four graphons averaged over 2000 replications.  The results generally agree with those shown in the figures.   The few relative discrepancies between RMSE and L1 errors occur when there is a small number of large errors, such as the boundary effects for the oracle for graphon 1, which affect RMSE more than the L1 error.  

For graphon 1, our method and the spectral clustering perform best.   For graphon 2, our method is only outperformed by universal singular value thresholding, whereas $n^{1/3}$ leading eigenvalue thresholding selects fewer eigenvalues than needed.  For graphon 3, our method is comparable to the $n^{1/3}$ leading eigenvalue thresholding, and they are both better than other methods, not counting the oracle.  For graphon 4, our method has shows significant advantage over all other methods except for the oracle.  Thus in all cases, our method shows very competitive performance compared to benchmarks.

\begin{table}[h!]
	\tbl{Root mean squared errors and mean absolute errors with standard errors, all multiplied by $10^2$, averaged over 2000 replications.  The largest relative error is less than $4\%$.}{
		\begin{tabular}{ccccccccc}\smallskip
			& \multicolumn{2}{c}{Graphon 1} 	& \multicolumn{2}{c}{Graphon 2}	& \multicolumn{2}{c}{Graphon 3} & \multicolumn{2}{c}{Graphon 4}\\
			& RMSE & MAE & RMSE & MAE & RMSE & MAE & RMSE & MAE\\
			Our method	
			&	1$\cdot$92	&	1$\cdot$33	&	3$\cdot$06	&	2$\cdot$25
			&	3$\cdot$00	&	1$\cdot$41	&	3$\cdot$55	&	2$\cdot$76\\
			\citet{chan2014consistent}	
			&	8$\cdot$78	&	3$\cdot$09	&	34$\cdot$17	&	30$\cdot$16
			&	11$\cdot$27	&	8$\cdot$04	&	4$\cdot$46	&	3$\cdot$58\\
			\citet{yang2014nonparametric}	
			&	9$\cdot$56	&	4$\cdot$14	&	34$\cdot$18	&	30$\cdot$19
			&	11$\cdot$47	&	8$\cdot$59	&	5$\cdot$67	&	4$\cdot$87\\
			$n^{1/3}$ singular value	
			&	2$\cdot$99	&	2$\cdot$25	&	4$\cdot$74	&	3$\cdot$59
			&	3$\cdot$16	&	1$\cdot$79	&	5$\cdot$86	&	4$\cdot$33\\
			Blockmodel spectral	
			&	1$\cdot$72	&	0$\cdot$75	&	33$\cdot$06	&	28$\cdot$80
			&	3$\cdot$98	&	1$\cdot$78	&	9$\cdot$08	&	6$\cdot$64\\
			Blockmodel oracle	
			&	5$\cdot$48	&	1$\cdot$42	&	5$\cdot$11	&	3$\cdot$80
			&	1$\cdot$62	&	0$\cdot$75	&	1$\cdot$06	&	0$\cdot$83\\
			\citet{chatterjee2014matrix}	
			&	4$\cdot$09	&	2$\cdot$25	&	1$\cdot$89	&	1$\cdot$47
			&	6$\cdot$39	&	3$\cdot$81	&	5$\cdot$67	&	4$\cdot$87\\
			Airoldi's method	
			&	15$\cdot$94	&	8$\cdot$92	&	15$\cdot$82	&	9$\cdot$23
			&	9$\cdot$40	&	5$\cdot$74	&	4$\cdot$60	&	3$\cdot$16\\
		\end{tabular}
	}
	\label{table::MSEs}
	\begin{tabnote}
		RMSE, root mean squared error; MAE, mean absolute error.
	\end{tabnote}
\end{table}

Overall, the results in this section show that various previously proposed methods can perform very well under their assumptions, which may be monotone degrees or low-rank or an underlying block model, but they fail when these assumptions are not satisfied.  Our method is the only one among those compared that performs well in a large range of scenarios, because it learns the structure from data via neighborhood selection instead of imposing a priori structural assumptions.   The $n^{1/3}$ singular value thresholding method also shows consistent performance across all graphons, although in all cases somewhat worse than ours.  It performs very well in the low-rank case, but if the leading singular values decay slowly, our method performs better.

\section{Application to link prediction} 
\label{section::linkprediction}

Evaluating probability matrix estimation methods on real networks directly is difficult, since the true probability matrix is unknown.  We assess the practical utility of our method by applying it to link prediction, a task that relies on estimating the probability matrix.    Here we think of the true adjacency matrix $A^{\textrm{true}}$ as unobserved, with binary edges drawn independently with probabilities given by $P$, also unobserved.   Instead we observe $A^{\textrm{obs}}_{ij} = M_{ij} A^{\textrm{true}}_{ij}$, where unobserved $M_{ij}$'s are independent Bernoulli$(1-p)$,  and $p$ is unknown.  Therefore $A^{\textrm{obs}}_{ij} = 1$ is always a true edge, but $A^{\textrm{obs}}_{ij} = 0$ could be either a true $0$ or a false negative.  This setting is different from and perhaps more realistic than the link prediction setting in \citet{gao2015optimal}, who assumed that $M_{ij}$'s are observed.  Under their setting, the missing rate $p$ can be estimated by the empirical missing rate $\hat{p}$, and all estimators can be corrected for missingness simply by dividing them by $1-\hat{p}$.

A link prediction method usually outputs a nonnegative score matrix $\hat{A}$, with scores giving the estimated propensity of a node pair to form an edge.   For methods that estimate the probability matrix, $\hat{A}$ can be taken to be $\hat P$;  other link prediction methods construct a binary $\hat A$ by working directly on $A$.    Both types of methods essentially output a ranked list of most likely missing links, useful in practice for follow-up confirmatory analysis.  

We compare various link prediction methods via their receiver operating characteristic curves. For each $t>0$, we define the false positive and the true positive rates by 
\begin{align*}
r_{\textrm{FP}}(t) &= {\sum_{ij}1\left(\hat{A}_{ij} > t, A^{\textrm{true}}_{ij}=0, M_{ij}=0\right)}\big/{\sum_{ij}1\left(A^{\textrm{true}}_{ij}=0, M_{ij}=0\right)}\\
r_{\textrm{TP}}(t) &= {\sum_{ij}1\left(\hat{A}_{ij} > t, A^{\textrm{true}}_{ij}=1, M_{ij}=0\right)}\big/{\sum_{ij}1\left(\hat{A}_{ij}=1, M_{ij}=0 \right)} \ .
\end{align*}
Then by varying $t$ we obtain the receiver operating characteristic curve.   In practice, $t$ is often selected to output a fixed number of most likely links.   

In this section we include three additional benchmark methods that produce score matrices rather than estimated probability matrices.   One standard score is the Jaccard index $\langle A_{i\cdot}, A_{j\cdot} \rangle \big/ \{(\sum_k A_{ik})(\sum_k A_{jk})\}$, see for example \citet{lichtenwalter2010new}.   The method by \cite{zhao2013link} computes scores so that similar node pairs to have similar predicted scores.  The PropFlow algorithm of \citet{lichtenwalter2010new} uses an expected random walk distance between nodes as the score.  

We first compare all methods on simulated networks generated from the graphons in Table \ref{table::fourgraphons}.  We set $n=500$ due to the computational cost of some of the benchmarks,  and set $p=10\%$.  All experiments are repeated 1000 times/ 
Figure 2 in the Supplementary Material shows the receiver operating characteristic curves for four graphons. Most differences between the methods can be inferred from Figure \ref{figure::simulation::benchmark::graphon_1}.  Overall, the methods based on graphon estimation outperform score-based methods.  Our method outperforms all other methods on this task, producing a receiver operating characteristic curve very close to that based on the true probability matrix $P$.

We also applied our method and others to the political blogs network \citep{adamic2005political}. This network consists of 1222 manually labelled blogs, 586 liberal and 636 conservative.  The network clearly shows two communities, with heterogeneous node degrees (there are hubs).    We removed 10\% of edges at random and calculated the receiver operating characteristic curve for predicting the missing links, shown in Figure \ref{figure::linkprediction::polblogs}.     Again, methods based on estimating the probability matrix performed much better than the scoring methods, and our method performs best overall.   Sort and smooth methods slightly outperformed spectral clustering and \citet{chatterjee2014matrix}, perhaps due to the presence of hubs.

\begin{figure}[h!]
	\centering
	\includegraphics[width=0.45\textwidth]{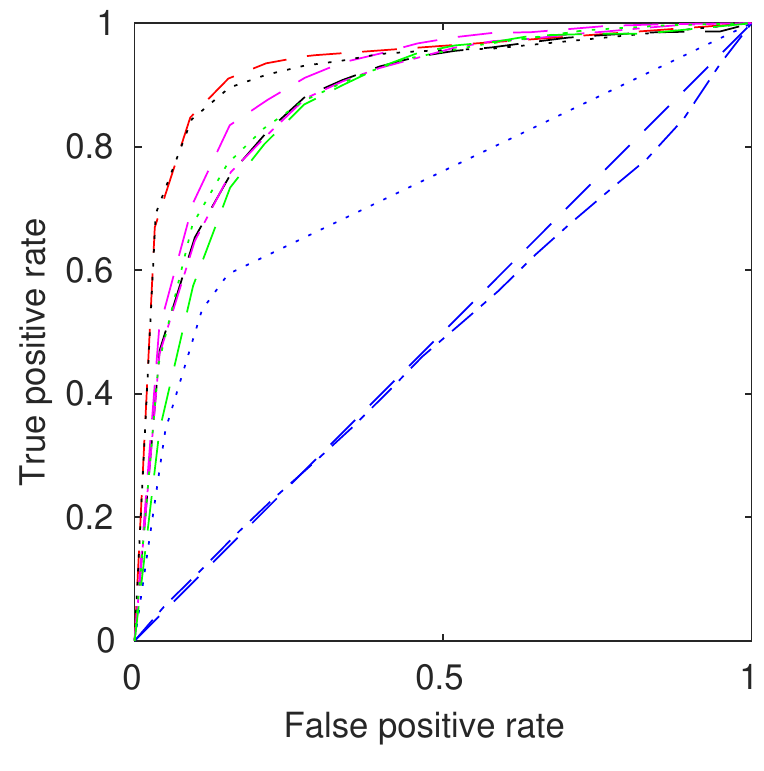}
	\vspace{-1em}
	\caption{Receiver operating characteristic curves for link prediction on the political blogs network. $10\%$ of edges are missing at random. Red dashed curve: our method;  black dotted curve: $n^{1/3}$ singular value thresholding; blue dashed curve: \citet{zhao2013link}; blue dash dotted curve: the Jaccard index; blue dotted curve: PropFlow; black dashed curve: \citet{chatterjee2014matrix}; magenta dashed curve: \citet{chan2014consistent}; magenta dashed dotted curve: \citet{yang2014nonparametric}; green dashed curve: block model with spectral clustering.}
	\label{figure::linkprediction::polblogs}
\end{figure}

\section{Discussion}\label{section::discussions}

The strength of our method is the adaptive neighborhood choice which works well under many different conditions;  it is also computationally efficient,  easy to implement, and essentially tuning free.  Its main limitation is the piecewise Lipschitz condition, which occasionally leads to over-smoothing.      Our method does not achieve the minimax error rate, and its rate cannot be improved; whether the minimax rate can be achieved by any polynomial time method is, to the best of our knowledge, an open problem.  Another major future challenge is relaxing the assumption of independent edges to better fit real-world networks. 

\section*{Acknowledgments}

The authors thank an associate editor and two anonymous referees for very helpful suggestions, and E. M. Airoldi of Harvard University for sharing a copy of his manuscript and code and great comments.  E.L. is supported by grants NSF DMS 1521551 and ONR N000141612910.  J.Z.  is supported by grants NSF DMS 1407698, KLAS 130026507, and KLAS 130028612.

\section*{Supplementary material}

The supplementary material includes numerical results on the bandwidth constant in Theorem \ref{theorem::errorrate},  $(2,\infty)$-norm errors and comparisons to benchmarks on link prediction for synthetic graphons from Section \ref{section::simulations}, and the proofs of Theorems \ref{theorem::errorrate} and \ref{proposition:sliceminimax} and Proposition \ref{proposition:errorrate}.

\section{Choosing the constant factor for the bandwidth}
\label{section::simulation::tuning}

First, we need to choose the quantile cut-off parameter $h$ which controls neighborhood selection. Theorem \ref{theorem::errorrate} gives the order of $h$, and the following numerical experiments empirically justify our choice of the constant factor.  Figure \ref{figure::tuning} shows the mean squared error curves for networks with $n = 2000$ nodes generated from the four graphons in Table \ref{table::fourgraphons}, with the constant factor $C$ varying in the range $\{2^{-3}, 2^{-2},\ldots,2^3\}$.
\begin{figure}[h]
	\centering
	\vspace{-7em}
	\includegraphics[width=0.6\textwidth]{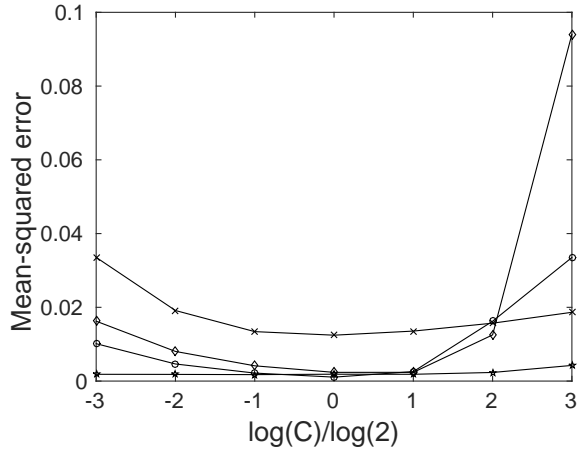}
	\vspace{-7em}
	\caption{Mean squared error of our method as a function of the constant $C$ in the tuning parameter $h = C \theerrorrate$. Graphons 1--4 are marked with a circle, an asterisk, a plus and a cross, respectively.}
	\label{figure::tuning}
\end{figure}

Figure \ref{figure::tuning} demonstrates that $C$ in the range from $2^{-2}$ to $2$ works equally well for all these very different graphons.    This suggests empirically that the method is robust to the choice of $C$, and therefore we set $C=1$ in all numerical results in the paper.

\section{Receiver operating characteristic curves for link prediction simulations in Section \ref{section::linkprediction}}

The results are presented in Figure \ref{figure::linkprediction::simulations}.    The link prediction results generally agree with how the methods performed on the task of estimating the probability matrix for these synthetic graphons, shown in Section 3 of the main manuscript.  

\begin{figure}[h]
	\centering
	\twoImages {0.45\textwidth}
	{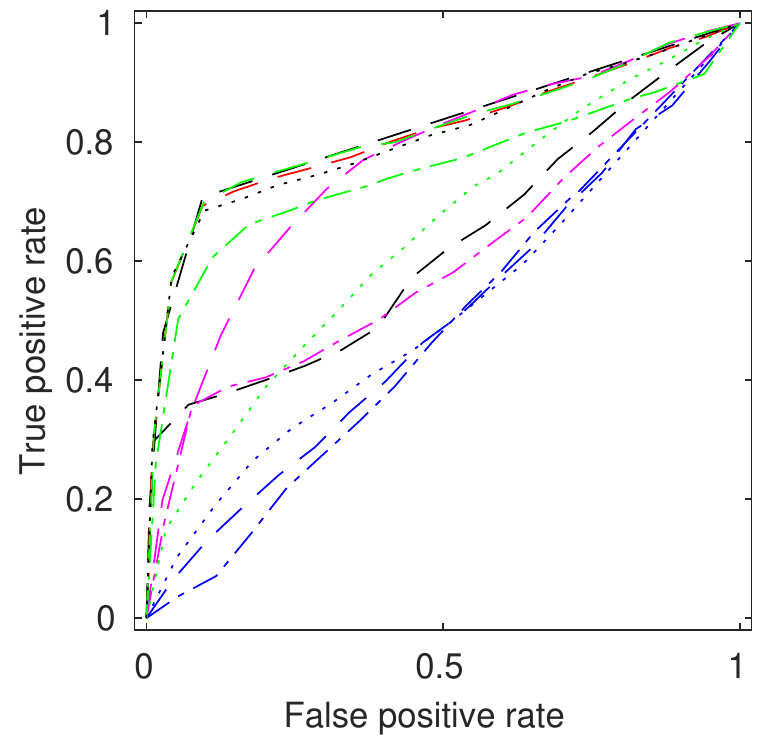}{Graphon 1}
	{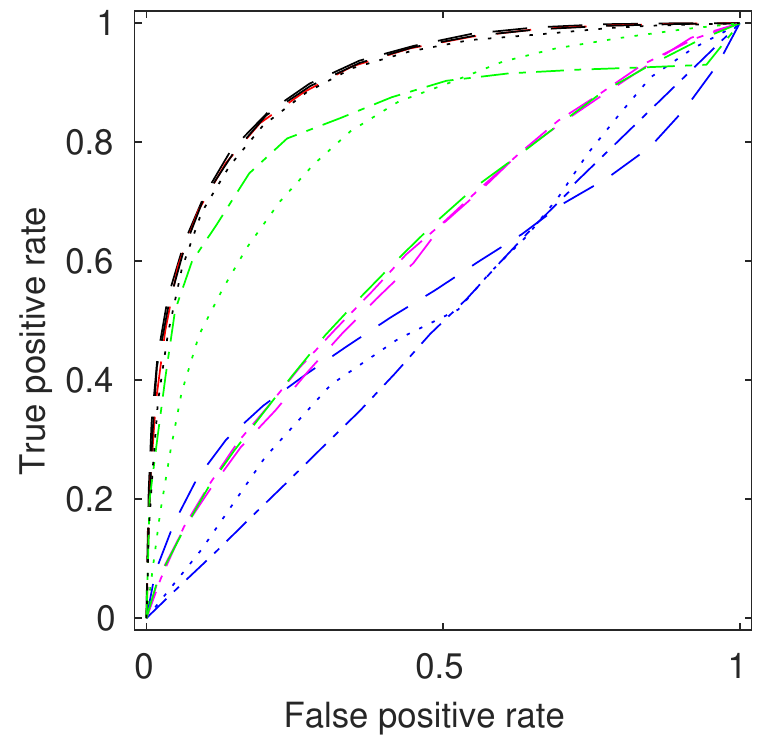}{Graphon 2}
	\twoImages {0.45\textwidth}
	{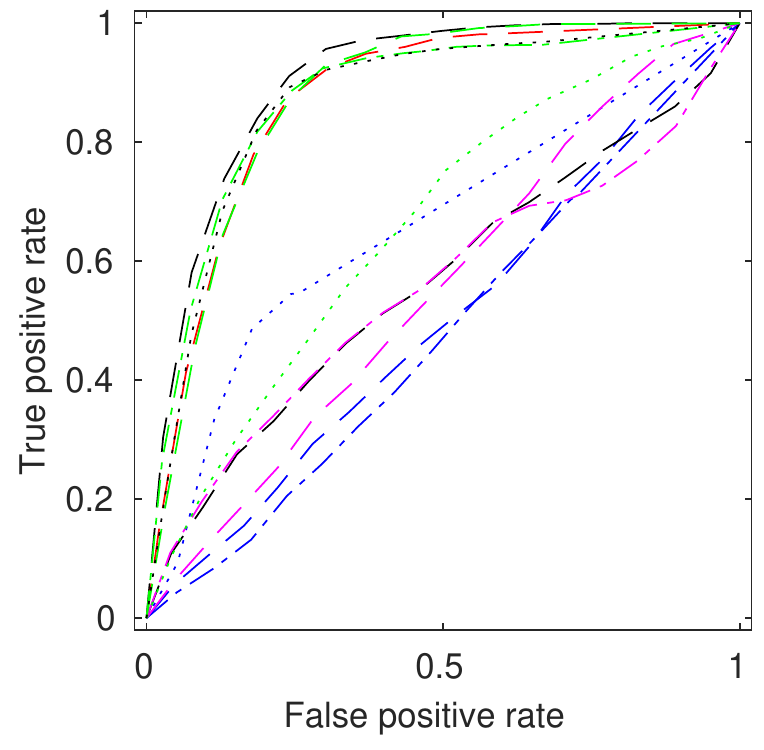}{Graphon 3}
	{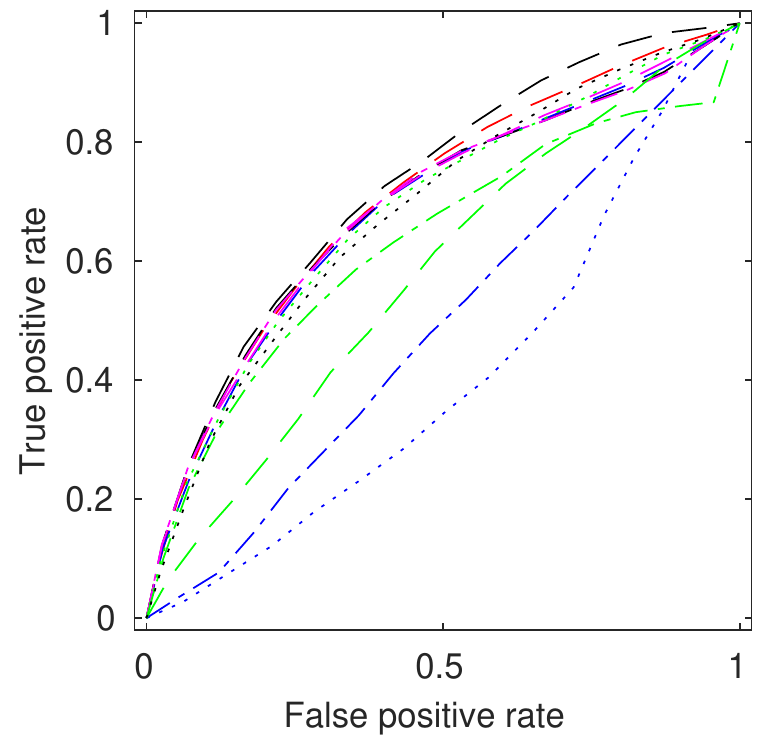}{Graphon 4}
	\caption{
		Receiver operating characteristic curves for link prediction by different methods for graphons 1 to 4.   Black dashed curve: the true probability matrix; red dashed curve: our method;  black dotted curve: $n^{1/3}$ singular value thresholding; blue dashed curve: \citet{zhao2013link}; blue dash dotted curve: the Jaccard index; blue dotted curve: PropFlow; black dashed curve: \citet{chatterjee2014matrix}; magenta dashed curve: \citet{chan2014consistent}; magenta dashed dotted curve: \citet{yang2014nonparametric}; green dashed curve: block model with spectral clustering; green dash dotted curve: oracle block model approximation.
	}\label{figure::linkprediction::simulations} 
\end{figure}

\section{Root mean squared $2,\infty$ errors for simulations in Section \ref{section::simulations}}

In Table \ref{table::MStwoinftys}, we report $\|\hat{P}-P\|_{2,\infty}/n^{1/2}$ for $\hat{P}$ estimated by all methods in Section \ref{section::simulations}.  The performances of our method is slightly better than but comparable to $n^{1/3}$ singular value thresholding and \citet{chatterjee2014matrix}, and these three methods are generally significantly better than other practical methods.  These results suggest there may be a $(2,\infty)$ error bound applicable to the low rank methods, even though they were not designed to control this type of errors;  investigating this is outside the scope of this manuscript.   

\begin{table}[h!]
	\caption{Root mean squared errors in the $(2,\infty)$ norm, followed by standard errors in parenthesis, all multiplied by $10^2$, averaged over 2000 replications.}
	\begin{tabular}{ccccc}\smallskip
		& Graphon 1 	& Graphon 2	& Graphon 3 & Graphon 4\\
		Our method      &       4$\cdot$84(0$\cdot$12)      &       5$\cdot$97(0$\cdot$08)      &       5$\cdot$60(0$\cdot$11)      &       7$\cdot$13(0$\cdot$11)\\
		\citet{chan2014consistent}      &       38$\cdot$87(0$\cdot$35)     &       50$\cdot$44(0$\cdot$80)     &       13$\cdot$74(0$\cdot$09)     &       9$\cdot$43(0$\cdot$14)\\
		\citet{yang2014nonparametric}   &       38$\cdot$87(0$\cdot$41)     &       49$\cdot$77(0$\cdot$74)     &       13$\cdot$11(0$\cdot$04)     &       7$\cdot$75(0$\cdot$07)\\
		$n^{1/3}$ singular value thresholding   &       6$\cdot$55(0$\cdot$18)      &       8$\cdot$34(0$\cdot$15)      &       5$\cdot$22(0$\cdot$11)      &       13$\cdot$67(0$\cdot$24)\\
		Blockmodel spectral     &       16$\cdot$65(1$\cdot$38)     &       45$\cdot$34(0$\cdot$39)     &       13$\cdot$00(0$\cdot$57)     &       29$\cdot$50(0$\cdot$53)\\
		Blockmodel oracle       &       35$\cdot$36(0$\cdot$04)     &       9$\cdot$47(0$\cdot$02)      &       2$\cdot$53(0$\cdot$02)      &       1$\cdot$73(0$\cdot$03)\\
		\citet{chatterjee2014matrix}    &       8$\cdot$77(0$\cdot$01)      &       3$\cdot$93(0$\cdot$10)      &       8$\cdot$42(0$\cdot$07)      &       7$\cdot$73(0$\cdot$06)\\
		Airoldi's method        &       34$\cdot$74(0$\cdot$46)     &       66$\cdot$97(0$\cdot$12)     &       16$\cdot$79(0$\cdot$07)     &       30$\cdot$47(0$\cdot$74)\\
	\end{tabular}
	
	\label{table::MStwoinftys}
\end{table}

\section{Proofs}
\begin{proof}[of Theorem \ref{theorem::errorrate}]
	For convenience, we start with summarizing notation and assumptions made in the main paper.  Let $0=x_0<x_1<\ldots<x_K=1$, $I_k:=[x_{k-1}, x_k)$ for $1\leq k \leq K-1$ and $I_K=[x_{K-1}, X_K]$. Assume the graphon $f$ is a bi-Lipschitz function on each of  $I_k\times I_\ell$ for $1\leq k,\ell\leq K$.  Let $L$ denote the maximum piece-wise bi-Lipschitz constant.
	\begin{assumption}
		The number of pieces $K$ may grow with $n$, as long as $\min_k |I_k| \big/ \theerrorrate\to \infty$.
	\end{assumption}

	For any $\xi\in[0,1]$, let $I(\xi)$ denote the $I_k$ that contains $\xi$.    Let $S_i(\Delta)=[\xi_i-\Delta,\xi_i+\Delta]\cap I(\xi_i)$ denote the neighborhood of $\xi_i$ in which $f(x, y)$ is Lipschitz in $x\in S_i(\Delta)$ for any fixed $y$. Finally, recall our estimator is defined by 
	{\color{black}
		$$
		\tilde{P}_{ij} = \frac{\sum_{i'\in \calN_i}A_{i'j}}{|\calN_i|}
		$$

		To prove the main theorem, it suffices to show that with high probability
		$$
		\frac{1}{n}\sum_j\left( \tilde{P}_{ij}-P_{ij} \right)^2 \leq \left(\frac{\log n}{n}\right)^{1/2}
		$$
		holds for all $i$.  We begin the proof with the following decomposition of the error term:
		\begin{align*}
		&\frac{1}{n}\sum_j\left( \tilde{P}_{ij}-P_{ij} \right)^2 = \frac{1}{n}\sum_j\left\{ \frac{\sum_{i'\in \calN_i}(A_{i'j}-P_{ij})}{|\calN_i|} \right\}^2\\
		=& \frac{1}{n}\sum_j \left[ \frac{\sum_{i'\in \calN_i}\left\{(A_{i'j} - P_{i'j}) + (P_{i'j}-P_{ij})\right\}}{|\calN_i|} \right]^2 \ . 
		\end{align*}
		
		We can bound the summand by
		\begin{align}
		&\left[ \frac{\sum_{i'\in \calN_i}\left\{(A_{i'j} - P_{i'j}) + (P_{i'j}-P_{ij})\right\}}{|\calN_i|} \right]^2 \nonumber\\
		\leq & 2\left\{\frac{\sumiprime(A_{i'j}-P_{i'j})}{|\calN_i|}\right\}^2 +  2\left\{\frac{\sumiprime(P_{i'j}-P_{ij})}{|\calN_i|}\right\}^2 = 2J_1(i,j) + 2J_2(i,j) \ . 
		\label{proof::furtherdecomposition}
		\end{align}}
	
	Our goal is to bound \tblue{$n^{-1}\sum_i\left\{J_1(i,j)+J_2(i,j)\right\}$}.   First, we prove a lemma which estimates the proportion of nodes in a diminishing neighborhood of $\xi_i$'s.
	\begin{lemma}\label{lemma::neighborhood}
		For arbitrary global constants $C_1$, $\tc_1>0$, define $\Delta_n=\left\{C_1+\left(\tc_1+4\right)^{1/2}\right\}\theerrorrate$.  For $n$ large enough so that 
		$\left\{(\tc_1+4)\log n/n\right\}^{1/2}\leq 1$ and $\Delta_n < \min_k|I_k|/2$,
		we have
		\begin{equation}
		\pr\left\{ \min_i\frac{|\{i'\neq i:\xi_{i'}\in S_i(\Delta_n)\}|}{n-1}\geq C_1\theerrorrate \right\} \geq 1-2n^{-\tc_1/4} . 
		\end{equation}
	\end{lemma}
	
	\begin{proof}[of Lemma \ref{lemma::neighborhood}]
		For any $0 < \epsilon \leq 1$ and $n$ large enough to satisfy the assumptions, by Bernstein's inequality we have, for any $i$, 
		\begin{align*}
		\pr\left\{ \left| \frac{|\{i'\neq i:\xi_{i'}\in S_i(\Delta_n)\}|}{n-1} - |S_i(\Delta_n)| \right| \geq \epsilon \right\} & \leq 2\exp\left\{-\frac{\frac{1}{2}(n-1)\epsilon^2}{1+\frac{1}{3}\epsilon}\right\} \leq 2\exp\left(-\frac{1}{4}n\epsilon^2\right) . 
		\end{align*}
		Taking a union bound over all $i$'s gives 
		\begin{equation*}
		\pr\left\{ \max_i \left| \frac{|\{i'\neq i:\xi_{i'}\in S_i(\Delta_n)\}|}{n-1} - |S_i(\Delta_n)| \right| \geq \epsilon \right\} \leq 2n\exp\left(-\frac{1}{4}n\epsilon^2\right) . 
		\end{equation*}	
		Letting $\epsilon=\left\{{(\tc_1+4)\log n}/{n}\right\}^{1/2}$, we have
		\begin{equation}
		\pr\left[ \max_i \left| \frac{|\{i'\neq i:\xi_{i'}\in S_i(\Delta_n)\}|}{n-1} - |S_i(\Delta_n)| \right| \geq \left\{\frac{(\tc_1+4)\log n}{n}\right\}^{1/2} \right] \leq 2n^{-\tc_1/4} .  \label{proof::lemma::proportionconcentration}
		\end{equation}
		
		Next we claim that either $[\xi_i-\Delta_n, \xi_i]\subseteq I(\xi_i)$ or $[\xi_i, \xi_i+\Delta_n]\subseteq I(\xi_i)$ holds for all $i$. If for some $i$ the claim does not hold, by the definition of $I(\xi_i)$, we have $I(\xi_i)\subset [\xi_i-\Delta_n, \xi_i+\Delta_n]$. So we have $|I(\xi_i)|\leq 2\Delta_n$, but this contradicts the condition $\Delta_n<\min_k |I_k|/2$. The claim yields that $|S_i(\Delta_n)|\geq \Delta_n$.
		Finally, by \eqref{proof::lemma::proportionconcentration}, with probability $1-2n^{-\tc_1/4}$, we have
		\begin{align*}
		\min_i\frac{|\{i'\neq i:\xi_{i'}\in S_i(\Delta_n)\}|}{n-1} &  \geq |S_i(\Delta_n)| - \left\{\frac{(\tc_1+4)\log n}{n}\right\}^{1/2} \\
		& \geq \Delta_n - \left\{\frac{(\tc_1+4)\log n}{n}\right\}^{1/2} \geq C_1\theerrorrate \ . 
		\end{align*}
		This completes the proof of Lemma \ref{lemma::neighborhood}.
	\end{proof}
	
	We now continue with the proof of Theorem \ref{theorem::errorrate}. Recall that we defined a measure of closeness of adjacency matrix slices in Section \ref{section::estimation} as 
	\begin{equation*}
	\tilde{d}(i,i') = \max_{k\neq i,i'}\left|\langle A_{i\cdot} - A_{i'\cdot}, A_{k\cdot} \rangle\right|\big/ n = \max_{k\neq i,i'}\left|(A^2/n)_{ik} - (A^2/n)_{jk}\right| \ . 
	\end{equation*}
	The neighborhood $\calN_i$ of node $i$ consists of nodes $(i')$'s with $\tilde{d}(i,i')$  below the $h$-th quantile of $\{\tilde{d}(i, k)\}_{k\neq i}$. 
	The next lemma shows two key properties of $\calN_i$.
	
	\begin{lemma}\label{lemma::propertyNi}
		Suppose we select the neighborhood $\calN_i$ by thresholding at the lower $h$-th quantile of $\{\tilde{d}(i, k)\}_{k\neq i}$, where we set $h=C_0\theerrorrate$ with an arbitrary global constant $C_0$ satisfying $0<C_0\leq C_1$ for the $C_1$ from Lemma \ref{lemma::neighborhood}. Let $C_2, \tc_2>0$ be arbitrary global constants and assume $n\geq 6$ is large enough so that
		\begin{inparaenum}[(i)]
			\item All conditions on $n$ in Lemma \ref{lemma::neighborhood} are satisfied;
			\item $\left\{(C_2+2)\log n/n\right\}^{1/2}\leq 1$;
			\item $C_1\left(n\log n\right)^{1/2}\geq 4$; and
			\item $4/n\leq \left\{\left(C_2+\tc_2+2\right)^{1/2} - \left(C_2+2\right)^{1/2} \right\}\theerrorrate$. \label{Ncondition::eliminate4divn}
		\end{inparaenum}
		Then the neighborhood $\calN_i$ has the following properties:
		\begin{enumerate}
			\item $|\calN_i|\geq C_0\left(n\log n\right)^{1/2}$.
			\item With probability $1-2n^{-\tc_1/4}-2n^{-\tc_2/4}$, for all $i$ and $i'\in\calN_i$, we have
			\begin{equation*}
			\|P_{i'\cdot}-P_{i\cdot}\|_2^2/n\leq \left[ 6L\left\{C_1+\left(\tc_2+4\right)^{1/2}\right\}^{1/2} + 8\left(C_2+\tc_2+2\right)^{1/2}  \right]\theerrorrate
			\end{equation*}
		\end{enumerate}
	\end{lemma}
	\begin{proof}[of Lemma \ref{lemma::propertyNi}]
		The first claim follows immediately from the choice of $h$ and  the definition of $\calN_i$.   
		To show the second claim, we start with concentration results. For any $i, j$ such that $i\neq j$, we have
		\begin{align}
		&\left|\left(A^2/n\right)_{ij}-\left(P^2/n\right)_{ij}\right|=\left| \sum_k \left(A_{ik}A_{kj} - P_{ik}P_{kj}\right) \right|\Big/n\nonumber\\
		\leq & \frac{|\sum_{k\neq i,j}(A_{ik}A_{kj} - P_{ik}P_{kj})|}{n-2}\cdot\frac{n-2}{n}  +  \frac{|(A_{ii}+A_{jj})A_{ij}| + |(P_{ii}+P_{jj})P_{ij}|}{n}\nonumber\\
		\leq & \frac{|\sum_{k\neq i,j}(A_{ik}A_{kj} - P_{ik}P_{kj})|}{n-2} + \frac{4}{n} \label{proof::lemma::elementwiseconcentration::1}
		\end{align}
		By Bernstein's inequality, for any $0 < \epsilon \leq 1$  and $n\geq 3$ we have
		\begin{equation*}
		\pr\left\{  \frac{|\sum_{k\neq i,j}(A_{ik}A_{kj} - P_{ik}P_{kj})|}{n-2} \geq \epsilon  \right\} \leq 2\exp\left\{ -\frac{\frac{1}{2}(n-2)\epsilon^2}{1+\frac{1}{3}\epsilon} \right\}\leq 2\exp\left(-\frac{1}{4}n\epsilon^2\right)  . 
		\end{equation*}
		Taking a union bound over all $i\neq j$, we have
		\begin{equation*}
		\pr\left\{  \max_{i, j: i\neq j}\frac{|\sum_{k\neq i,j}(A_{ik}A_{kj} - P_{ik}P_{kj})|}{n-2} \geq \epsilon  \right\} \leq 2n^2\exp\left(-\frac{1}{4}n\epsilon^2\right)  . 
		\end{equation*}
		Then setting $\epsilon = \left\{(C_2+2)\log n/n\right\}^{1/2}$ with $n$ large enough so that $\epsilon\leq 1$, we have
		\begin{equation}
		\pr\left\{  \max_{i,j:i\neq j}\frac{|\sum_{k\neq i,j}(A_{ik}A_{kj} - P_{ik}P_{kj})|}{n-2} \geq   \left\{\frac{(C_2+2)\log n}{n}\right\}^{1/2}\right\} \leq 2n^{-\tc_2/4}  \label{proof::lemma::elementwiseconcentration::2}
		\end{equation}
		Combining \eqref{proof::lemma::elementwiseconcentration::1} and \eqref{proof::lemma::elementwiseconcentration::2}, with probability $1-2n^{-\tc_2/4}$, the following holds
		\begin{equation}
		\max_{i,j:i\neq j}\left|\left(A^2/n\right)_{ij}-\left(P^2/n\right)_{ij}\right|  \leq   \left\{\frac{(C_2+2)\log n}{n}\right\}^{1/2} + \frac{4}{n}\leq \left\{\frac{(C_2+\tc_2+2)\log n}{n}\right\}^{1/2} \label{proof::lemma::elementwiseconcentrationA2}
		\end{equation}
		for $n$ large enoug to satisfy \eqref{Ncondition::eliminate4divn}. 
		Next, we prove a useful inequality.   For all $i$ and any $\ti$ such that $\xi_{\ti}\in S_i(\Delta_n)$, we have
		\begin{equation}
		\left|\left(P^2/n\right)_{ik}-\left(P^2/n\right)_{\ti k}\right| = |\langle P_{i\cdot}, P_{k\cdot} \rangle  -  \langle P_{\ti\cdot}, P_{k\cdot} \rangle|/n\leq \|P_{i\cdot} - P_{\ti\cdot}\|_2\|P_{k\cdot}\|_2/n\leq L\Delta_n  \label{proof::lemma::Papproximation}
		\end{equation}
		for all $k$, where the last inequality follows from 
		$$|P_{i'\ell} - P_{i\ell}| = |f(\xi_{i'}, \xi_\ell) - f(\xi_i,\xi_\ell)|\leq L|\xi_{i'} - \xi_i|\leq L\Delta_n$$
		for all $\ell$, and $\|P_{k\cdot}\|_2\leq n^{1/2}$ for all $k$. Note that this holds for all $k$, including $k = i$ or $k = \ti$.
		
		We are now ready to upper bound $\tilde{d}(i, i')$ for $i'\in\calN_i$. We bound $\tilde{d}(i, i')$ via bounding $\tilde{d}(i, \ti)$ for $\ti$ with $\xi_{\ti}\in S_i(\Delta_n)$. By \eqref{proof::lemma::elementwiseconcentrationA2} and \eqref{proof::lemma::Papproximation}, with probability $1-2n^{-\tc_2/4}$, we have
		\begin{align}
		&\tilde{d}(i, \ti) =\max_{k\neq i, \ti}|(A^2/n)_{ik} - (A^2/n)_{\ti k}| \nonumber\\
		\leq& \max_{k\neq i, \ti}|(P^2/n)_{ik} - (P^2/n)_{\ti k}| + 2\max_{i, j: i\neq j}|(A^2/n)_{ij} - (P^2/n)_{ij}|\nonumber\\
		\leq&  L\Delta_n+2\left\{\frac{(C_2+\tc_2+2)\log n}{n}\right\}^{1/2}  \label{proof::bounddtilda1}
		\end{align}
		
		Now since the fraction of nodes contained in $\left|\{ \ti:\xi_{\ti}\in S_i(\Delta_n) \}\right|$ is at least $h$, this upper bounds $\tilde{d}(i,i')$ for $i'\in\calN_i$, since nodes in $\calN_i$   have the lowest $h$ fraction of values in $\{\tilde{d}(i,k)\}_k$.	Setting $\Delta_n$ as in Lemma \ref{lemma::neighborhood}, by Lemma \ref{lemma::neighborhood} and \eqref{proof::lemma::elementwiseconcentrationA2}, with probability $1-2n^{-\tc_1/4}-2n^{-\tc_2/4}$, for all $i$, at least $C_1\theerrorrate$ fraction of nodes $\ti\neq i$ satisfy both $\xi_{\ti}\in S_i(\Delta_n)$ and
		\begin{equation}
		\tilde{d}(i,\ti)\leq L\Delta_n+2\left\{\frac{(C_2+\tc_2+2)\log n}{n}\right\}^{1/2} \ .  \label{proof::bounddtilda2}
		\end{equation}
		Recall that $i' \in \calN_i$  have the smallest $h=C_0\theerrorrate\leq C_1\theerrorrate$ fraction of $\tilde{d}(i,i')$'s. Then \eqref{proof::bounddtilda2} yields that
		\begin{equation}
		\tilde{d}(i,i')\leq L\Delta_n+2\left\{\frac{(C_2+\tc_2+2)\log n}{n}\right\}^{1/2}  \label{proof::lemma::boundtildedfinal}
		\end{equation}
		holds for all $i$ and all $i'\in\calN_i$ simultaneously with probability $1-2n^{-\tc_1/4}-2n^{-\tc_2/4}$.
		
		We are now ready to complete the proof of the second claim of Lemma \ref{lemma::propertyNi}. By Lemma \ref{lemma::neighborhood}, \eqref{proof::lemma::elementwiseconcentrationA2}, \eqref{proof::lemma::Papproximation} and \eqref{proof::lemma::boundtildedfinal}, with probability $1-2n^{-\tc_1/4}-2n^{-\tc_2/4}$, the following holds. For $n$ large enough such that $\min_i\left|\left\{ i':\xi_{i'}\in S_i(\Delta_n) \right\}\right|\geq C_1\left(n\log n\right)^{1/2}\geq 4$ (by Lemma  \ref{lemma::neighborhood}), for all $i$ and $i'\in\calN_i$ we can find $\ti\in S_i(\Delta_n)$ and $\tip\in S_{i'}(\Delta_n)$ such that $i$, $i'$, $\ti$ and $\tip$ are different from each other. Then we have
		\begin{align*}
		&\|P_{i\cdot} - P_{i'\cdot}\|_2^2/n = (P^2/n)_{ii}-(P^2/n)_{i'i} + (P^2/n)_{i'i'} - (P^2/n)_{ii'}\\
		\leq & \left|(P^2/n)_{ii}-(P^2/n)_{i'i}\right| + \left|(P^2/n)_{i'i'} - (P^2/n)_{ii'}\right|\\
		\leq & \left|(P^2/n)_{i\ti}-(P^2/n)_{i'\ti}\right| + \left|(P^2/n)_{i'\tip}-(P^2/n)_{i\tip}\right| + 4L\Delta_n\\
		\leq & \left|(A^2/n)_{i\ti}-(A^2/n)_{i'\ti}\right| + \left|(A^2/n)_{i'\tip}-(A^2/n)_{i\tip}\right| + 4\left\{\frac{(C_2+\tc_2+2)\log n}{n}\right\}^{1/2} + 4L\Delta_n\\
		\leq & 2\max_{k\neq i, i'}\left|(A^2/n)_{ik} - (A^2/n)_{i'k}\right|  + 4\left\{\frac{(C_2+\tc_2+2)\log n}{n}\right\}^{1/2} + 4L\Delta_n\\
		= & 2\tilde{d}(i, i') + 4\left\{\frac{(C_2+\tc_2+2)\log n}{n}\right\}^{1/2} + 4L\Delta_n \leq 8\left\{\frac{(C_2+\tc_2+2)\log n}{n}\right\}^{1/2} + 6L\Delta_n\\
		= & \left[ 6L\left\{C_1+\left(\tc_2+4\right)^{1/2}\right\}^{1/2} + 8\left(C_2+\tc_2+2\right)^{1/2}  \right]\theerrorrate \ .
		\end{align*}
		This completes the proof of Lemma \ref{lemma::propertyNi}.
	\end{proof}
	
	We are now ready to bound \tblue{$n^{-1}\sum_{j}\{J_1(i,j)+J_2(i,j)\}$}, which will complete the proof of Theorem \ref{theorem::errorrate}. Note that we cannot simply bound each individual $J_1(i,j)$'s by Bernstein's inequality since $A_{i'j}$ is not independent of the event $i'\in\calN_i$.   Instead, we work with the sum $n^{-1}\sum_j J_1(i, j)$ and decompose it as follows.
	\begin{align}
	&\frac{1}{n}\sum_j J_1(i, j) = \frac{1}{n|\calN_i|^2}\sum_j\left\{ \sumiprime (A_{i'j} - P_{i'j}) \right\}^2\nonumber\\
	= & \frac{1}{n|\calN_i|^2}\sum_j\left\{ \sumiprime (A_{i'j} - P_{i'j})^2 + \sumiprime\sum_{i''\neq i', i''\in\calN_i}(A_{i'j}-P_{i'j})(A_{i''j}-P_{i''j}) \right\} . 
	\label{proof::decomposeJ1}
	\end{align}
	The first term in \eqref{proof::decomposeJ1} satisfies 
	\begin{align}
	\sum_j (A_{i'j} - P_{i'j})^2/n &= \|A_{i'\cdot} - P_{i'\cdot}\|_2^2/n \leq 1  \label{proof::boundJ1further1}
	\end{align}
	where the inequality is due to $|A_{i'j}-P_{i'j}|\leq 1$ for all $j$. The second term in \eqref{proof::decomposeJ1} can be bounded by
	\begin{align}
	\frac{1}{n|\calN_i|^2} & \sum_j \sumiprime\sum_{i''\neq i', i''\in\calN_i}(A_{i'j}-P_{i'j})(A_{i''j}-P_{i''j}) \leq  \nonumber \\
	\leq & \frac{1}{|\calN_i|^2} \sum_{i',i''\in\calN_i: i'\neq i''} \left| \frac{1}{n}\sum_j(A_{i'j}-P_{i'j})(A_{i''j}-P_{i''j}) \right|\nonumber\\
	\leq &  \frac{1}{|\calN_i|^2} \sum_{i',i''\in\calN_i: i'\neq i''} \Bigg\{ \frac{1}{n-2}\left|\sum_{j\neq i', i''}(A_{i'j}-P_{i'j})(A_{i''j}-P_{i''j})\right|\cdot\frac{n-2}{n}  \nonumber\\
	& +  \frac{\left|(A_{i'i''}-P_{i'i''})\right|\left| (A_{i'i'}-P_{i'i'}+A_{i''i''}-P_{i''i''}) \right|}{n} \Bigg\}\nonumber\\
	\leq & \frac{1}{|\calN_i|^2} \sum_{i',i''\in\calN_i: i'\neq i''} \left\{  \frac{1}{n-2}\left|\sum_{j\neq i', i''}(A_{i'j}-P_{i'j})(A_{i''j}-P_{i''j})\right| + \frac{2}{n}  \right\} . 
	\label{proof::boundJ1further2}
	\end{align}
	To bound the first term in \eqref{proof::boundJ1further2}, for any $i_1\neq i_2$ and $\epsilon>0$, by Bernstein's inequality we have
	\begin{align*}
	\pr\left\{ \frac{1}{n-2}\left| \sum_{j\neq i_1, i_2} \left( A_{i_1 j} - P_{i_1 j} \right) \left( A_{i_2 j} - P_{i_2 j} \right) \right| \geq \epsilon \right\} 	\leq 2\exp\left\{ -\frac{\frac{1}{2}(n-2)\epsilon^2}{1+\frac{1}{3}\epsilon} \right\} \leq 2n^2 e^{-n\epsilon^2/4} . 
	\end{align*}
	Let $C_3, \tc_3>0$ be arbitrary global constants and let $n$ be large enough so that $1/\{C_0\left(n\log n\right)^{1/2}\} + 2/n\leq \left\{ \left(C_3+\tc_3+8\right)^{1/2} - \left(C_3+8\right)^{1/2} \right\}\theerrorrate$. First, taking $\epsilon = \left\{(C_3+8)\log n/n\right\}^{1/2}$ and a union bound over all  $i_1\neq i_2$, we have
	\begin{equation}
	\pr\left[ \max_{i_1, i_2, i_1\neq i_2}\frac{1}{n-2}\left| \sum_{j\neq i_1, i_2} \left( A_{i_1 j} - P_{i_1 j} \right) \left( A_{i_2 j} - P_{i_2 j} \right) \right| \geq \left\{\frac{(C_3+8)\log n}{n}\right\}^{1/2} \right] \leq 2n^{-\tc_3/4} . 
	\label{proof::boundtermJ1}
	\end{equation}
	Then plugging \eqref{proof::boundJ1further1}, \eqref{proof::boundJ1further2} and \eqref{proof::boundtermJ1} into \eqref{proof::decomposeJ1} and combining with claim 1 of Lemma \ref{lemma::propertyNi}, with probability $1-2n^{-\tc_1/4}-2n^{\tc_2/4}-2n^{-\tc_3/4}$, for all $i$ simultaneously, we have
	\begin{align}
	&\frac{1}{n}\sum_j J_1(i, j) \leq \frac{1}{|\calN_i|^2}\sumiprime\left[ 1 + (|\calN_i|-1)\left(\left\{\frac{(C_3+8)\log n}{n}\right\}^{1/2}  + \frac{2}{n}\right)\right] \nonumber\\
	\leq & \frac{1}{|\calN_i|} + \left\{\frac{(8+C_3)\log n}{n}\right\}^{1/2} + \frac{2}{n} \leq \frac{1}{C_0\left(n\log n\right)^{1/2}} + \frac{2}{n} + \left\{\frac{(C_3+8)\log n}{n}\right\}^{1/2}\nonumber\\
	\leq& \left\{\frac{(C_3+\tc_3+8)\log n}{n}\right\}^{1/2} . 
	\label{proof::boundJ1final}
	\end{align}
	
	We now bound \tblue{$n^{-1}\sum_{j}J_2(i,j)$}. By Lemma \ref{lemma::propertyNi}, with probability $1-2n^{-\tc_1/4}-2n^{\tc_2/4}$, \tblue{for all $i$ simultaneously,} we have
	{\color{black}\begin{align}
		&\frac{1}{n}\sum_{j}J_2(i,j)=\frac{1}{n}\sum_j\left\{ \frac{\sumiprime(P_{i'j}-P_{ij})}{|\calN_i|} \right\}^2\nonumber\\
		\leq& \frac{\sumiprime \sum_j (P_{i'j}-P_{ij})^2/n}{|\calN_i|} = \frac{\sumiprime \|P_{i'\cdot} - P_{i\cdot}\|_2^2/n}{|\calN_i|}\nonumber\\
		\leq& \left[ 6L\left\{C_1+\left(\tc_2+4\right)^{1/2}\right\}^{1/2} + 8\left(C_2+\tc_2+2\right)^{1/2}  \right]\theerrorrate \ , 
		\label{proof::boundJ2final}
		\end{align}}
	where the first inequality is the Cauchy-Schwartz inequality and the second inequality follows from claim 2 of Lemma \ref{lemma::propertyNi}.
	
	Combining \eqref{proof::boundJ1final} and \eqref{proof::boundJ2final} completes the proof of Theorem \ref{theorem::errorrate}.
	
\end{proof}

\begin{proof}[of Proposition \ref{proposition:errorrate}]
	We only need to prove the upper bound on the error rate;  the lower bound follows by setting $\delta=1$ and applying \tblue{Theorem 2.3} in \citet{gao2014rate}.   We show that the least squares estimator defined in (2.4) of \citet{gao2014rate}, \tblue{which we shall denote as $\hat{P}_{\textrm{LS}}$ here,} has the same error rate for graphons in our family ${\cal F}_{\delta, L}$.  The proof is almost identical to the proof of \tblue{Theorem 2.3} in \citet{gao2014rate}, except we need to choose a $\theta^*$ that respects the partition of $[0,1]$ into intervals of continuity with a non-essential adaptation of their Lemma 2.1. Referring to the proof of Lemma 2.1, instead of choosing $z^*$ using $U_a=\left[ (a-1)/k, a/k \right)$, we now set
	$$
	(z^*)^{-1}(a) = \{ i\in[n]: \xi_i\in \tilde{U}_a \}
	$$
	where $\tilde{U}_a$'s are defined as follows.  Recall Definition 2 of our paper which specifies the sequence $x_0, \ldots, x_K$.   We set $\delta =\min_{0\leq s\leq K-1}(x_{s+1}-x_s) \succ \theerrorrate$ and consequently $K\prec (n/\log n)^{1/2}$. Like in \citet{gao2014rate}, we use a stochastic block model with $k =n^{1/2}$ equal-sized communities to approximate the true probability matrix. This corresponds to using a piece-wise constant graphon function, with pieces of equal size $1/k=n^{-1/2}$ to approximate the true graphon. \tblue{Thus for large enough $n$, we have the following properties: at most one $x_i$, $1\leq i\leq K-1$ may fall in any $U_a$ for $1\leq a\leq k$; no $x_i$ may fall in $U_1$, $U_2$, $U_{n-1}$ or $U_k$; and if an $x_i$ falls in $U_a$, no $x_i$ will fall in $U_{a-2}$, $U_{a-1}$, $U_{a+1}$ or $U_{a+2}$, because for large enough $n$, we have}
	{\color{black}\begin{align*}
		x_i-x_{i-1}&\geq \delta\asymp (n^{-1}\log n)^{1/2}\succ n^{-1/2}\asymp 3/k =  |U_{a-2} \cup U_{a-1} \cup U_a| \ , \\
		x_{i+1}-x_i&\geq \delta\asymp (n^{-1}\log n)^{1/2}\succ n^{-1/2}\asymp 3/k =  |U_a \cup U_{a+1} \cup U_{a+2}|\ .
		\end{align*}
		Then we define $\tilde{U}$ as follows.  First, we set $\tilde{U}_1=U_1$ and $\tilde{U}_n=U_n$.  For all $2\leq a\leq n-1$ such that no $x_i$, $1\leq i\leq K-1$ falls in any of $U_{a-1}$, $U_a$ and $U_{a+1}$, we let $\tilde{U}_a=U_a$.  Lastly for all $2\leq a\leq n-1$ such that an $x_i$, $1\leq i\leq K-1$ falls in $U_a$, if $x_1\leq (2a-1)/(2k)$, we set $\{ \tilde{U}_{a-1}, \tilde{U}_a, \tilde{U}_{a+1} \} = \{ U_{a-1}\cup \left[ (a-1)/k, x_i \right), \left[ x_i, a/k \right), U_{a+1} \}$, and otherwise we set $\{ \tilde{U}_{a-1}, \tilde{U}_a, \tilde{U}_{a+1} \} = \{ U_{a-1}, \left[ (a-1)/k, x_i \right), \left[ x_i, a/k \right)\cup U_{a+1} \}$.
		
		Let $P^*$ denote the probability matrix for the stochastic block model approximation to the true probability matrix $P$ corresponding to the partition $\tilde{U}$.  That is, for all $i, j$ such that $\xi_i\in \tilde{U}_a$ and $\xi_j\in \tilde{U}_b$, define:
		$$
		P^*_{ij} = \frac{1}{|\{i': \xi_{i'}\in \tilde{U}_a\}| |\{j': \xi_{j'}\in \tilde{U}_b\}|}\sum_{i'\in \tilde{U}_a, j'\in\tilde{U}_b} f(\xi_{i'}, \xi_{j'})
		$$
		We can apply the reasoning in the proof of Lemma 2.1 in \citet{gao2014rate}, which yields:
		\begin{align}
		\left| P^*_{ij} - P_{ij} \right| &\leq \frac{1}{|\{i': \xi_{i'}\in \tilde{U}_a\}| |\{j': \xi_{j'}\in \tilde{U}_b\}|}\sum_{i'\in \tilde{U}_a, j'\in\tilde{U}_b} \left| f(\xi_{i'}, \xi_{j'}) - f(\xi_i,\xi_j)\right|\nonumber\\
		&\leq \frac{1}{|\{i': \xi_{i'}\in \tilde{U}_a\}| |\{j': \xi_{j'}\in \tilde{U}_b\}|}\sum_{i'\in \tilde{U}_a, j'\in\tilde{U}_b} L\left( |\xi_i-\xi_{i'}| + |\xi_j-\xi_{j'}| \right)\nonumber\\
		&\leq 4L/k\asymp 1/k=n^{-1/2} \label{Lemma::new_lemma_2.1} \ .
		\end{align}
		We can then apply the argument in Equations (4.1) through (4.5) in \citet{gao2014rate} and combine our result \eqref{Lemma::new_lemma_2.1} with Lemmas 4.1, 4.3 and 4.4 in \citet{gao2014rate} to obtain the following result:
		$$
		\|\hat{P}_{\textrm{LS}}-P\|_F^2 \leq 2\|\hat{P}_{\textrm{LS}} - P^*\|_F^2 + 2\|P^*-P\|_F^2 \leq C(k^2+k\log n+ n^2k^{-2}) \leq C' n\log n \ ,
		$$
		where we plugged in $k=n^{1/2}$.
	}
\end{proof}

\begin{proof}[of Theorem \ref{proposition:sliceminimax}]
	{\color{black}In this proof, we first construct a ``baseline'' network to be a stochastic block model with $K=3$ communities, which we label $0$, $1$, and $2$, of sizes $m$, $\ell$, and $\ell$, respectively.  Thus $n=m+2\ell$.  
		Letting $Z^{(0)}\in\{0, 1\}^{n\times K}$ denote the membership matrix of the $n$ nodes, we can, without loss of generality, set $Z_{i1}=1$ for $1\leq i\leq m$, $Z_{i2}=1$ for $m+1\leq i\leq m+\ell$ and $Z_{i3}=1$ for $m+\ell+1\leq i\leq n$.  For all other $(i, k)$, we set $Z_{ik}=0$.  Next, we define the $K \times K$ probability matrix $B$ as:
		$$
		B = \begin{pmatrix}
		1/2 & 1/2+\phi & 1/2\\
		1/2+\phi & 1/2 & 1/2\\
		1/2 & 1/2 & 1/2
		\end{pmatrix}
		$$
		where $\phi$ is a small positive value to be determined later.  The probability matrix  is then $EA = P = ZBZ^T$.
		
		We then construct $N$ probability matrices ${\cal P} = \{P^{(1)},\ldots,P^{(N)}\}$, where $N$ is a natural number to be determined later.  Each $P^{(i)}$ is a stochastic block model with $K=3$ communities,its own community membership assignment, and the same $B$ as the baseline network.  That is, $P^{(i)}=Z^{(i)}B\left\{Z^{(i)}\right\}^T$.
		
		Now we construct $Z^{(i)}$.  Without loss of generality, we assume $\ell$ is an even number.  Otherwise, we can switch one node from communities $1$ and $2$ each to community $0$, and this will not change the lower bound on the error rate we are going to establish.  To proceed, we use the Gilbert-Varshamov bound, which was also used by \citet{gao2014rate} and \citet{klopp2015oracle}, to construct $N$ indicator vectors $w^{(1)},\ldots,w^{(N)}$:. 
		\begin{lemma} \label{lemma::Varshamov_Gilbert}
			For any positive integer $\ell$, there exists $\{ w^{(1)},\ldots,w^{(N)} \}$ for some $N\geq \exp(\ell/8)$, where $w^{(i)}\in\{0,1\}^\ell$, such that for any $i\neq j$, we have
			$$
			\|w^{(i)}-w^{(j)}\|_H\geq \ell/4
			$$
			where $\|x-y\|_H = \sum_s 1(x_s\neq y_s)$ is the Hamming distance.
		\end{lemma}
		We construct $Z^{(i)}$ as follows.  For $1\leq j\leq m+\ell/2$ and $m+\ell+1\leq j\leq m+3\ell/2$, we set $Z^{(i)}_{j\cdot} = Z^{(0)}_{j\cdot}$.  That is, for all nodes in community 0 and the first half of nodes in communities 1 and 2, their community memberships match the baseline network.  For $1\leq j'\leq \ell/2$, set $Z^{(i)}_{m+\ell/2+j'\cdot}=(0, 1, 0)$ if $w^{(i)}_{j'}=0$, and set $Z^{(i)}_{m+\ell/2+j'\cdot}=(0, 0, 1)$ if $w^{(i)}_{j'}=1$.  For $\ell/2+1\leq j'\leq \ell$, set $Z^{(i)}_{m+\ell+j'\cdot}=(0, 0, 1)$ if $w^{(i)}_{j'}=0$, and set $Z^{(i)}_{m+\ell+j'\cdot}=(0, 1, 0)$ if $w^{(i)}_{j'}=1$.  That is, for the second half of community 1 and 2, the community membership matches the baseline network if the corresponding element in $w$ is $0$; otherwise 1 and 2 are switched.   Therefore, for any $i\neq j\in\{1,\ldots,N\}$, we have
		\begin{align*}
		d_\inftt^2\left\{ P^{(i)}, P^{(j)} \right\}  &  \geq \|P^{(i)}_{1\cdot} - P^{(j)}_{1\cdot}\|_2^2 / n  = \phi^2\|w^{(i)}-w^{(j)}\|_H / n \geq \frac{\phi^2 \ell}{4n}    
		\end{align*}
		We will choose $\phi$ small enough for a sufficiently large $n$ such that $\phi\leq 1/4$, so that every edge probability lies in $(1/2,3/4)$, which enables us to apply Proposition 4.2 from \cite{gao2014rate}.  Then noticing that for $i\neq j$, matrices $P^{(i)}$ and $P^{(j)}$ can only differ by $\phi$ in at most $2mn$ elements by definition, we have
		\begin{align*}
		D\left\{ P^{(i)} || P^{(j)} \right\}  &  \leq 8\|P^{(i)}-P^{(j)}\|_F^2  \leq 8\cdot 2 mn\cdot\phi^2 = 16mn\phi^2   
		\end{align*}
		where $D(f||g)$ denotes the Kullback-Leibler divergence between distributions $f$ and $g$.  
		
		Now we are ready to complete the proof.  By Lemma 3 of \cite{yu1997assouad}, we have
		\begin{align}
		\max_i E_{P^{(i)}} d_\inftt^2\left\{ \hat{P}, P^{(i)} \right\} &\geq \frac{\min_{i,j}d_\inftt^2\left\{ P^{(i)}, P^{(j)} \right\}}{4}\left[ 1-\frac{\max_{i,j}D\left\{ P^{(i)} || P^{(j)} \right\}+2}{\log |\mathcal{P}|} \right]^2\nonumber\\
		& \geq  \frac{\phi^2\ell}{16n}\left(1-\frac{16mn\phi^2+2}{\ell/8}  \right)^2 \geq \frac{\phi^2\ell}{16n}\left(\frac12-\frac{128mn\phi^2}{\ell}  \right)^2 \label{equation::lower_bound_finish}
		\end{align}
		The we maximize the right hand side of \eqref{equation::lower_bound_finish} by setting $\phi^2=\ell/(257mn)$ and $\ell\geq \max\{n/3, 32\}$.  Then
		$$
		\frac{\phi^2\ell}{16n}\left(\frac{1}{2}-\frac{128mn\phi^2}{\ell}  \right)^2 = \frac{\ell^2}{8224mn^2}\geq \frac{1}{74016m}
		$$
		The matrices ${\cal P}=\{P^{(1)},\ldots,P^{(N)}\}$ need to corresopnd to stochastic block models in our piece-wise bi-Lipschitz graphon space $F_{\delta, L}$ such that $\delta\Big/ (\log n/n)^{1/2}\to\infty$.  That is, the size of the smallest block must grow to infinity faster than $(n\log n)^{1/2}$.  Therefore,
		\begin{equation*}
		\max_{f\in F_{\delta, L}: \delta/ (\log n/n)^{1/2}\to\infty} 1/m = (n\log n)^{-1/2}
		\end{equation*}
		This completes the proof.}
\end{proof}

\bibliographystyle{biometrika}
\bibliography{manuscript_3}

\begin{thebibliography}{27}
\expandafter\ifx\csname natexlab\endcsname\relax\def\natexlab#1{#1}\fi

\bibitem[{Adamic \& Glance(2005)}]{adamic2005political}
\textsc{Adamic, L.~A.} \& \textsc{Glance, N.} (2005).
\newblock The political blogosphere and the 2004 {U.S. Election:} divided they
  blog.
\newblock In \textit{Proceedings of the 3rd International Workshop on Link
  Discovery}, LinkKDD '05. New York: ACM, pp. 36--43.

\bibitem[{Airoldi et~al.(2008)Airoldi, Blei, Fienberg \&
  Xing}]{airoldi2009mixed}
\textsc{Airoldi, E.~M.}, \textsc{Blei, D.~M.}, \textsc{Fienberg, S.~E.} \&
  \textsc{Xing, E.~P.} (2008).
\newblock Mixed membership stochastic blockmodels.
\newblock \textit{Journal of Machine Learning Research} \textbf{9}, 1981--2014.

\bibitem[{Airoldi et~al.(2013)Airoldi, Costa \& Chan}]{airoldi2013stochastic}
\textsc{Airoldi, E.~M.}, \textsc{Costa, T.~B.} \& \textsc{Chan, S.~H.} (2013).
\newblock Stochastic blockmodel approximation of a graphon: Theory and
  consistent estimation.
\newblock In \textit{Advances in Neural Information Processing Systems 26},
  C.~J.~C. Burges, L.~Bottou, M.~Welling, Z.~Ghahramani \& K.~Q. Weinberger,
  eds. Red Hook, NY: Curran Associates, Inc., pp. 692--700.

\bibitem[{Aldous(1981)}]{aldous1981representations}
\textsc{Aldous, D.~J.} (1981).
\newblock Representations for partially exchangeable arrays of random
  variables.
\newblock \textit{Journal of Multivariate Analysis} \textbf{11}, 581--598.

\bibitem[{Amini et~al.(2013)Amini, Chen, Bickel \& Levina}]{amini2013pseudo}
\textsc{Amini, A.~A.}, \textsc{Chen, A.}, \textsc{Bickel, P.~J.} \&
  \textsc{Levina, E.} (2013).
\newblock Pseudo-likelihood methods for community detection in large sparse
  networks.
\newblock \textit{The Annals of Statistics} \textbf{41}, 2097--2122.

\bibitem[{Amini \& Levina(2017)}]{amini2014semidefinite}
\textsc{Amini, A.~A.} \& \textsc{Levina, E.} (2017).
\newblock On semidefinite relaxations for the block model.
\newblock \textit{The Annals of Statistics} To appear.

\bibitem[{Bickel \& Chen(2009)}]{bickel2009nonparametric}
\textsc{Bickel, P.~J.} \& \textsc{Chen, A.} (2009).
\newblock A nonparametric view of network models and {N}ewman--{G}irvan and
  other modularities.
\newblock \textit{Proceedings of the National Academy of Sciences}
  \textbf{106}, 21068--21073.

\bibitem[{Cai \& Li(2015)}]{cai2015robust}
\textsc{Cai, T.~T.} \& \textsc{Li, X.} (2015).
\newblock Robust and computationally feasible community detection in the
  presence of arbitrary outlier nodes.
\newblock \textit{The Annals of Statistics} \textbf{43}, 1027--1059.

\bibitem[{Candes \& Plan(2011)}]{candes2011tight}
\textsc{Candes, E.~J.} \& \textsc{Plan, Y.} (2011).
\newblock Tight oracle inequalities for low-rank matrix recovery from a minimal
  number of noisy random measurements.
\newblock \textit{IEEE Transactions on Information Theory} \textbf{57},
  2342--2359.

\bibitem[{Chan \& Airoldi(2014)}]{chan2014consistent}
\textsc{Chan, S.~H.} \& \textsc{Airoldi, E.} (2014).
\newblock A consistent histogram estimator for exchangeable graph models.
\newblock \textit{Journal of Machine Learning Research Workshop and Conference
  Proceedings} \textbf{32}, 208--216.

\bibitem[{Chatterjee(2015)}]{chatterjee2014matrix}
\textsc{Chatterjee, S.} (2015).
\newblock Matrix estimation by universal singular value thresholding.
\newblock \textit{The Annals of Statistics} \textbf{43}, 177--214.

\bibitem[{Chaudhuri et~al.(2012)Chaudhuri, Chung \&
  Tsiatas}]{chaudhuri2012spectral}
\textsc{Chaudhuri, K.}, \textsc{Chung, F.} \& \textsc{Tsiatas, A.} (2012).
\newblock Spectral clustering of graphs with general degrees in the extended
  planted partition model.
\newblock \textit{Journal of Machine Learning Research} \textbf{2012}, 1--23.

\bibitem[{Choi(2017)}]{choi2015co}
\textsc{Choi, D.} (2017).
\newblock Co-clustering of nonsmooth graphons.
\newblock \textit{The Annals of Statistics} \textbf{45}, 1488--1515.

\bibitem[{Choi \& Wolfe(2014)}]{choi2014co}
\textsc{Choi, D.} \& \textsc{Wolfe, P.~J.} (2014).
\newblock Co-clustering separately exchangeable network data.
\newblock \textit{The Annals of Statistics} \textbf{42}, 29--63.

\bibitem[{Gao et~al.(2016)Gao, Lu, Ma \& Zhou}]{gao2015optimal}
\textsc{Gao, C.}, \textsc{Lu, Y.}, \textsc{Ma, Z.} \& \textsc{Zhou, H.~H.}
  (2016).
\newblock Optimal estimation and completion of matrices with biclustering
  structures.
\newblock \textit{Journal of Machine Learning Research} \textbf{17}, 1--29.

\bibitem[{Gao et~al.(2015)Gao, Lu \& Zhou}]{gao2014rate}
\textsc{Gao, C.}, \textsc{Lu, Y.} \& \textsc{Zhou, H.~H.} (2015).
\newblock Rate-optimal graphon estimation.
\newblock \textit{The Annals of Statistics} \textbf{43}, 2624--2652.

\bibitem[{Gu{\'e}don \& Vershynin(2016)}]{guedon2014community}
\textsc{Gu{\'e}don, O.} \& \textsc{Vershynin, R.} (2016).
\newblock Community detection in sparse networks via {G}rothendieck's
  inequality.
\newblock \textit{Probability Theory and Related Fields} \textbf{165},
  1025--1049.

\bibitem[{Hoff(2008)}]{hoff2008modeling}
\textsc{Hoff, P.} (2008).
\newblock Modeling homophily and stochastic equivalence in symmetric relational
  data.
\newblock \textit{Advances in Neural Information Processing Systems}
  \textbf{20}, 657--664.

\bibitem[{Karrer \& Newman(2011)}]{karrer2011stochastic}
\textsc{Karrer, B.} \& \textsc{Newman, M.~E.} (2011).
\newblock Stochastic blockmodels and community structure in networks.
\newblock \textit{Physical Review E} \textbf{83}, 016107--1--10.

\bibitem[{Klopp et~al.(2017)Klopp, Tsybakov \& Verzelen}]{klopp2015oracle}
\textsc{Klopp, O.}, \textsc{Tsybakov, A.~B.} \& \textsc{Verzelen, N.} (2017).
\newblock Oracle inequalities for network models and sparse graphon estimation.
\newblock \textit{The Annals of Statistics} \textbf{45}, 316--354.

\bibitem[{Lichtenwalter et~al.(2010)Lichtenwalter, Lussier \&
  Chawla}]{lichtenwalter2010new}
\textsc{Lichtenwalter, R.~N.}, \textsc{Lussier, J.~T.} \& \textsc{Chawla,
  N.~V.} (2010).
\newblock New perspectives and methods in link prediction.
\newblock In \textit{Proceedings of the 16th ACM SIGKDD International
  Conference on Knowledge Discovery and Data Mining}, KDD '10. New York: ACM,
  pp. 243--252.

\bibitem[{Olhede \& Wolfe(2014)}]{olhede2014network}
\textsc{Olhede, S.~C.} \& \textsc{Wolfe, P.~J.} (2014).
\newblock Network histograms and universality of blockmodel approximation.
\newblock \textit{Proceedings of the National Academy of Sciences}
  \textbf{111}, 14722--14727.

\bibitem[{Rohe et~al.(2011)Rohe, Chatterjee \& Yu}]{rohe2011spectral}
\textsc{Rohe, K.}, \textsc{Chatterjee, S.} \& \textsc{Yu, B.} (2011).
\newblock Spectral clustering and the high-dimensional stochastic blockmodel.
\newblock \textit{The Annals of Statistics} \textbf{39}, 1878--1915.

\bibitem[{Saade et~al.(2014)Saade, Krzakala \&
  Zdeborov\'{a}}]{saade2014spectral}
\textsc{Saade, A.}, \textsc{Krzakala, F.} \& \textsc{Zdeborov\'{a}, L.} (2014).
\newblock Spectral clustering of graphs with the bethe hessian.
\newblock In \textit{Advances in Neural Information Processing Systems 27},
  Z.~Ghahramani, M.~Welling, C.~Cortes, N.~D. Lawrence \& K.~Q. Weinberger,
  eds. Red Hook, NY: Curran Associates, Inc., pp. 406--414.

\bibitem[{Yang et~al.(2014)Yang, Han \& Airoldi}]{yang2014nonparametric}
\textsc{Yang, J.~J.}, \textsc{Han, Q.} \& \textsc{Airoldi, E.~M.} (2014).
\newblock Nonparametric estimation and testing of exchangeable graph models.
\newblock \textit{Journal of Machine Learning Research, Workshop and Conference
  Proceedings} \textbf{33}, 1060--1067.

\bibitem[{Yu(1997)}]{yu1997assouad}
\textsc{Yu, B.} (1997).
\newblock Assouad, {F}ano, and {L}e {C}am.
\newblock In \textit{Festschrift for Lucien Le Cam}. New York: Springer, pp.
  423--435.

\bibitem[{Zhao et~al.(2017)Zhao, Wu, Levina \& Zhu}]{zhao2013link}
\textsc{Zhao, Y.}, \textsc{Wu, Y.-J.}, \textsc{Levina, E.} \& \textsc{Zhu, J.}
  (2017).
\newblock Link prediction for partially observed networks.
\newblock \textit{Journal of Computational and Graphical Statistics} To appear.

\end{thebibliography}

\end{document}